\documentclass[11pt,letter]{article}
\usepackage[margin=1in]{geometry}

\usepackage[utf8]{inputenc}

\usepackage{algorithm}
\usepackage{algorithmic}

\usepackage[colorlinks=true,
            linkcolor=blue,
            urlcolor=blue,
            citecolor=blue]{hyperref}

\usepackage{amsmath}
\usepackage{amssymb}
\usepackage{mathtools}
\usepackage{amsthm}

\PassOptionsToPackage{}{natbib}
\usepackage[round]{natbib}

\usepackage{enumitem}

\usepackage[dvipsnames]{xcolor}
\usepackage[]{color-edits}

\newcommand{\A}{\mathcal{A}}

\newcommand{\M}{\mathcal{M}\mathcal{G}}
\renewcommand{\S}{\mathcal{S}}
\renewcommand{\A}{\mathcal{A}}
\newcommand{\G}{\mathcal{G}}
\newcommand{\F}{\mathcal{F}}

\renewcommand{\L}{\mathcal{L}}
\newcommand{\N}{\mathcal{N}}

\newcommand{\R}{\mathbb{R}}

\newcommand{\E}{\mathbb{E}}
\newcommand{\Pe}{\mathbb{P}}
\renewcommand{\P}{\mathcal{P}}

\newcommand{\Z}{\mathbb{Z}}

\newcommand{\norm}[1]{\left\|#1\right\|}

\newcommand{\bre}{\mathrm{br}}

\newcommand{\dist}{\operatorname{dist}}

\newtheorem{theorem}{Theorem}[section]
\newtheorem{definition}{Definition}[section]

\newtheorem{lemma}{Lemma}[section]
\newtheorem{proposition}{Proposition}[section]

\newtheorem{remark}{Remark}[section]

\usepackage{smile}

\newcommand{\titolo}{Finite-sample Guarantees for Nash Q-learning with Linear Function Approximation}
\newcommand{\algname}{Nash Q-learning with optimistic value iteration}
\newcommand{\algbrev}{NQOVI}

\newcommand\blfootnote[1]{%
  \begingroup
  \renewcommand\thefootnote{}\footnote{#1}%
  \addtocounter{footnote}{-1}%
  \endgroup
}

\title{
\titolo
}

\author{Pedro Cisneros-Velarde${}^{1}$\quad\quad Sanmi Koyejo${}^{2}$ \\\\
$^{1}${University of Illinois at Urbana-Champaign}\\
$^{2}${Stanford University, Google Research}
}
\date{}

\begin{document}
\maketitle

\begin{abstract}
Nash Q-learning may be considered one of the first and most known algorithms in multi-agent reinforcement learning (MARL) for learning policies that constitute a Nash equilibrium of an underlying general-sum Markov game. 
Its original proof provided asymptotic guarantees and was for the tabular case. Recently, finite-sample guarantees have been provided 
using more modern RL techniques for the tabular case.  
Our work analyzes Nash Q-learning 
using 
linear function approximation -- a representation regime introduced when the state space is large or continuous -- and provides finite-sample guarantees that indicate its sample efficiency. We find that the obtained performance nearly matches an existing efficient result for single-agent RL under the same representation and has a polynomial gap when compared to the best-known result for the tabular case. 
\end{abstract}

\blfootnote{Emails: Pedro Cisneros-Velarde (pacisne@gmail.com) and Sanmi Koyejo (sanmi@cs.stanford.edu).}


\section{Introduction}
\label{sec:introduction}

Multi-agent reinforcement learning (MARL) has been successfully applied to a diversity of problems, such as solving the games of Go~\citep{alphaGo:16, alphaZero:17} and Starcraft~\citep{vinyals2019grandmaster}, coordination of unmanned aerial vehicles~\citep{pham2018cooperative}, autonomous driving~\citep{Dinneweth2022SurveyCar}, power systems~\citep{foruzan18Microgrid}, and management of water and energy resources~\citep{Ni2014WaterResource,Lingxiao2020Energy}. The theory and development of 
multi-agent reinforcement learning algorithms is currently a prolific area, as attested by various recent surveys on the field, e.g.,~\citep{zhang2021multi,HernandezLeal2019Survey,Yang2021Survey}. In general, employing MARL to solve for a Nash equilibrium general-sum Markov game is computationally complex~\citep{Daskalakis2019Complexity}. This motivated theoretical works to look for other weaker solution concepts (e.g., coarse-correlated equilibria), or, if looking for a Nash equilibrium, either: (i) leave the general-sum domain and focus on zero-sum games or fully cooperative games, or (ii) specify extra conditions for the underlying general-sum Markov game (MG)~\citep{zhang2021multi}. The seminal work~\citep{Hu2003NashQ} introduced the \emph{Nash Q-learning} algorithm in the context of infinite-horizon discounted Markov games. The idea of Nash Q-learning 
is that, at every time step, each agent needs to find a Nash equilibrium which solves some static game whose utilities or rewards are defined 
by 
the (estimates of the) Q-functions of all the 
agents -- this is also called a \emph{stage game}. In~\citep{Hu2003NashQ}, asymptotic learning guarantees are given when the chosen Nash equilibrium is consistent in all stage games and is either a \emph{global optimal} or a \emph{saddle} one. 
Despite this strong sufficient condition,~\cite{Hu2003NashQ} present numerical examples where Nash Q-learning solves games that do not satisfy such conditions. 
It is important to remark that there exist proven cases in which value-based methods -- encompassing Nash Q-learning -- cannot converge to a single Nash equilibrium of general-sum Markov games~\citep{Zinkevich2005CyclicEquilibria}. However, remarkably, Nash Q-learning stands as one of the few general-sum MARL algorithms and has elicited the development of algorithms specialized to other classes of Markov games or focused on other solution concepts. Further, it is still consistently cited in the applied literature~\citep{HernandezLeal2019Survey}.  

The first 
formal proof for Nash Q-learning  by~\cite{Hu2003NashQ} only provided formal guarantees for asymptotic convergence in the tabular setting. However, recently, about two decades later, ~\cite{Liu2021SharpSelfPlay} proposed a type of Nash Q-learning algorithm and used a modern approach 
from the theoretical reinforcement learning (RL) literature to establish finite-sample guarantees and thus guarantee the sample efficiency of learning in the tabular setting. ~\cite{Liu2021SharpSelfPlay} used regret as a performance metric, and thus it was of interest that the average performance of policies gets closer to the performance of a Nash equilibrium instead of an actual convergence to a single equilibrium.

In the modern RL literature, it is known that tabular approaches are not ideal in environments where the state space is large or continuous. This has motivated the development of \emph{linear function approximation}, where, for example, the transition kernel and reward function of the underlying Markov decision process (MDP) are a linear function of a vector of features~\citep{CJ-ZY-ZW-MIJ:20, Yang2020RL}.

Taken together, this prior work motivates the central question of our paper:  

\emph{Can we obtain finite-sample guarantees and sample efficiency for Nash Q-learning in the linear function approximation regime?} 

We answer this question positively by proposing a Nash Q-learning algorithm -- called \emph{\algname} (\algbrev) -- and providing its finite sample guarantees under a regret performance metric. Interestingly, we find that the sample efficiency of our algorithm nearly matches the one reported in~\citep{CJ-ZY-ZW-MIJ:20} for (single-agent) RL in the same approximation regime.  

In general, our central question is also motivated from the fact that an increasing number of works providing sample efficient guarantees for (single-agent) RL problems has appeared in recent years.
The works~\citep{CJ-ZY-ZW-MIJ:20, Yang2020RL} started providing such guarantees in the linear function approximation domain using the principle of \emph{optimism} under uncertainty for \emph{online} RL -- later, other works have applied it to \emph{reward-free} RL (e.g.~\citep{RW-SD-LY-RS:20}) and have even applied a counterpart principle, called \emph{pessimism}, to \emph{offline} RL~\citep{Jin2021Pevi}. 
Optimism consists of adding a bonus so that the estimated optimistic Q-function rewards more those state-action pairs that have been less explored. 
Pessimism basically does the opposite by subtracting a bonus value. However, to the best of our knowledge, the simultaneous application of optimism and pessimism to achieve sample efficiency for learning Nash equilibria (in online MARL) has mainly been limited to two-player zero-sum games in the linear function approximation case~\citep{SQ-JY-ZW-ZY:22}, and to general-sum games in the tabular case~\citep{Liu2021SharpSelfPlay}. In this work, we show that the principle of optimism can easily be applied to Nash Q-learning in general-sum games.

\paragraph{Contributions} We summarize our contributions.
\begin{itemize}
    \item We provide the first sample efficient guarantees for a Nash Q-learning algorithm in the linear function approximation regime for general-sum games -- obtaining a regret bound $\tilde{\cO}(\sqrt{Kd^3H^5})$, with $K$ being the number of episodes, $H$ the episode length, and $d$ the dimension of the feature vector of the linear function approximation.
    \item To prove our guarantees, we propose the \algname~(\algbrev) algorithm.
    The original Nash Q-learning proposed by~\cite{Hu2003NashQ} was in the context of tabular and discounted MGs, and considered convergence to a Nash equilibrium as a performance metric. In contrast, we consider episodic MGs with regret performance, and do no need the existence of special Nash equilibria on the stage games as in~\cite{Hu2003NashQ}.
    \item When directly transforming it to the tabular case, our performance bound has a polynomial gap on all factors except for the number of episodes $K$ compared to the best-known result by~\cite{Liu2021SharpSelfPlay}. 
    \item 
    In the single agent case, our \algbrev ~algorithm collapses to the model-free RL algorithm proposed by~\cite{CJ-ZY-ZW-MIJ:20} (instead of taking a (mixed) Nash equilibrium at each stage game, the agent takes the optimal greedy action). Remarkably, we show that our algorithm's sample efficiency differs only by a factor of $H$ -- the length of the episode -- compared to the single agent one. To the best of our knowledge, this is the first time a general-sum MARL algorithm nearly matches the sample efficiency of an RL algorithm. 
\end{itemize}

\subsection{Related Works}

Since our paper is of a theoretical nature, we limit ourselves to presenting prior work focused on theory.

{\par \textbf{Multi-agent RL (MARL)}} 
Although the applied MARL literature has been around for decades, theoretical works have been gaining more presence in recent years -- we refer the reader to the recent surveys~\citep{zhang2020model,HernandezLeal2019Survey,Yang2021Survey}. 
Importantly, we highlight that a large body of 
recent works have focused on the study of learning in
the two-player zero-sum Markov game case -- where one player tries to maximize the expected reward while the other tries to minimize it. One reason for its popularity is that it can be formulated as a minimax game and Nash equilibria are easily characterized~\citep{zhang2021multi}. Recent works have been done both in the tabular setting, e.g.,~\citep{kozuno2021learning,zhang2020model,Bai2020SelfPlay,Liu2021SharpSelfPlay,jin2022vlearning}, and the linear function approximation setting, e.g.,~\citep{chen2022almost,Cisneros-Velarde2022OnePolicyEnough,SQ-JY-ZW-ZY:22}. 
In the case of general-sum Markov games, another large body of work has focused on providing guarantees for finding other solution concepts such as coarse correlated equilibria (CCE); e.g.,~\citep{Liu2021SharpSelfPlay,jin2022vlearning,Mao202REfficientRLGeneralSum}. Minimax sample optimality 
has been shown -- under certain assumptions -- for finding CCE in general-sum games and Nash equilibria in zero-sum games for the tabular case~\citep{li2022minimaxoptimal}. 
In learning Nash equilibria, \cite{Liu2021SharpSelfPlay} proposed a Nash Q-learning algorithm for general-sum games in the tabular setting, with an underlying episodic MG -- no extra conditions on the Nash equilibria are required.
While writing our paper we found the recent preprint by~\citep{chengzhuo2022RLGSMG} which studies representation learning in general-sum games 
and their proposed algorithms 
output a policy after a number of episodes. 
They focus on the harder problem of learning the feature vector of the linear approximation, whereas we assume it is given -- we only focus on learning a good policy and not on learning a good representation. 
Thus our guarantees are not directly comparable. 
Finally, we remark that both~\cite{Liu2021SharpSelfPlay} and~\cite{chengzhuo2022RLGSMG} use the principle of optimism and pessimism, so they compute two Q-functions on their algorithms, while we compute just one optimistic Q-function.

{\par \textbf{Linear function approximation in RL}} The idea of using linear function approximation is ubiquitous in theoretical RL. The first works to combine it with 
optimism for sample efficient learning were~\citep{CJ-ZY-ZW-MIJ:20, Yang2020RL} for online RL. 
Since then, such setting has been adapted to different RL problems, such as representation learning (of the feature vector of the linear approximation), e.g.~\citep{agarwal2020flambe}; parallel learning (multiple agents learning through independent MDPs but being able to communicate their experience), e.g.,~\citep{dubey2021provably}; 
deployment efficiency (RL algorithms when the number of times a policy can be updated is restricted), e.g.,~\citep{MG-RX-SSD-LFY:21}; 
reward-free RL (where exploration and exploitation are separated in different learning stages), 
e.g.,~\citep{RW-SD-LY-RS:20,AW-YC-MS-SSD-KJ:22}. 
Some works have combined two of the aforementioned problems 
using 
linear function approximation; e.g. in the context of reward-free RL, ~\citet{huang2021deployment} studied deployment efficiency, whereas~\citet{Cisneros-Velarde2022OnePolicyEnough} studied the effect of parallel exploration. 
These works follow a similar skeleton in their algorithms since all of them have in common the use of optimism and value iteration -- it is in this framework that we decided to propose an algorithm based on Nash Q-learning.

The paper is organized as follows. In Section~\ref{sec:preliminaries}, we formally introduce the setting. In Section~\ref{sec:NashQ-analysis}, we introduce our Nash Q-learning algorithm and state our main result. In Section~\ref{sec:NashQproof}, we provide a sketch of the proof and some nuances of its formal analysis. Section~\ref{sec:conclusion} is the conclusion.
\subsection{Notation} 
Let $\|\cdot\|$ be the Euclidean norm, and $\|v\|_A = \sqrt{v^TAv}$ for positive semidefinite matrix $A$. 
Let $[k]=\{1, 2, \dots, k\}$ for a positive integer $k$. Let $I_m$ be the $m \times m$ identity matrix. Let $\Delta(\A)$ be the probability simplex defined on a given finite set $\A$. 
%
Given the big-O complexity notation $\cO$, we use $\tilde{\cO}$ to hide polylogarithmic terms in the quantities of interest.


\section{Preliminaries}
\label{sec:preliminaries}

We consider an episodic Markov Game (MG) of the form $\M=(\S,\A,H,\P,r,)$, with state space $\S$, action space $\A=\A_1\times\dots\A_n$ with $\A_i$ being the action space for agent $i\in[n]$, $H$ is the number of steps per episode or episode length -- we assume the non-bandit case $H\geq 2$, $\P = \{\P_h\}_{h\in[H]}$ are transition probability measures and $\P_h(\cdot\mid x,a)$ denotes the transition kernel over $h+1$ if all players take the action profile $a\in\A$ for state $x\in\S$. 
We denote agent $i$'s reward function profile by $r_i = \{ r^i_h\}_{h=1}^H$ with $r^i_h: \cS \times \A_i \to [0,1]$.\footnote{We assume deterministic rewards for simplicity.} 
For any agent $i\in[n]$, its action taken at step $h$ is denoted by $a_{i,h}\in\A_i$, 
and let 
$a_i = \{ a_{i,h}\}_{h=1}^H$. 
We assume every agent has a finite action space, while the state space can be arbitrarily large or even continuous. 

We denote agent $i$'s policy by $\pi_i = \{ \pi_{i,h}\}_{h=1}^H$ with $\pi_{i,h}: \cS \to \Delta(\A_i)$. 
With some abuse of notation, we also let $\pi_h:\S\to\Delta(\A)$ denote the (joint) policy taken by the agents over the joint action space at time step $h\in[H]$ -- the subindex $k$ in $\pi_k$ will be clear from the context whether it refers to an agent or a time step. 
Let $\pi$ be the joint policy of all agents. 
We say $\pi$ is a product policy (across agents) $\pi=\pi_1\times\dots\times\pi_n$ when, conditioned on the same state, the action of every agent can be sampled independently according to their own policy, i.e., $\pi_h(x)\in \Delta(\A_1)\times\dots\times\Delta(\A_n)$ for every $x\in\S$, $h\in[H]$.
For agent $i$, we define her value function $V_h^{i,\pi}: \S \to \R$ at the $h$-th step as $V_h^{i,\pi}(x) = \E_{\pi}[\sum_{h'=h}^H r_{h'} (s_{h'}, a_{h'})  \given s_h = x]$ and her Q-function or action-value function $Q_h^{i,\pi}: \S \times \A \to \R$ as $Q_h^{i,\pi}(x,a):=\E_{\pi}[ \sum_{h'=h}^H r_{h'}(s_{h'}, a_{h'}, b_{h'} )  \given s_h = x, a_h = a]$, where the expectation $\mathbb{E}_{\pi}$ is taken with respect to both the randomness in the transitions $\P$ and the randomness inherent in the policy $\pi$. If agent $i$ has policy $\nu$ and the rest of the agents joint policy $\pi_{-i}$, we denote its associated value function at step $h$ by the notation $V_h^{i,\nu,\pi_{-i}}$, i.e., by placing a superscript with $i$'s policy before the (joint) policy of the rest of the agents; consequently,  $V_h^{i,\pi_i,\pi_{-i}}=V_h^{i,\pi}$. 

We now define our solution concept of interest.

\begin{definition}[Nash equilibrium for Markov games]
\label{def:NashEq-MG}
Given an initial state $s_o\in\cS$, a product policy profile $\pi^*$ is called a \emph{Nash equilibrium} (NE) if $V^{i,\pi^*_i,\pi^*_{-i}}_1(s_o)\geq V^{i,\pi_i,\pi^*_{-i}}_1(s_o)$ for any $i\in[n]$ and any policy $\pi_i$, and it is called an \emph{$\epsilon$-NE} if $V^{i,\pi^*_i,\pi^*_{-i}}_1(s_o)- V^{i,\pi_i,\pi^*_{-i}}_1(s_o)\leq\epsilon$. 
\end{definition}
We say agent $i\in[n]$ plays a \emph{best response} policy against the policy profile $\pi_{-i}$ of the rest of the agents according to  $\bre_i(\pi_{-i})\in\argmax_{\nu} V_h^{i,\nu, \pi_{-i}}(x)$ for any $(x,h)\in\S\times[H]$. Note that we can easily characterize a Nash equilibrium (Definition~\ref{def:NashEq-MG}) using best-responses.

For any function $f: \mathcal{S} \to \mathbb{R}$, we define the transition operator as $(\mathbb{P}_{h}f)(x, a) =\E_{x'\sim\P_h(\cdot|x,a)}[f(x')]$ and the Bellman operator as $(\mathbb{B}_{h}f)(x, a) = r_{h} +  (\mathbb{P}_{h}f)(x, a)$ for each step $h \in [H]$. For any $i\in[n]$, the Bellman equation associated with a policy $\pi$ is: $Q^{i,\pi}_h(x,a)=(r_h^i(x,a)+\Pe_h V^{i,\pi}_{h+1})(x,a)$, $V^{i,\pi}_h(x)= \E_{a\sim \pi_h(x)}[Q^{i,\pi}_h(x,a)]$, with $V^{i,\pi}_{H+1}(x)=0$, for any $(x,a)\in\S\times\A$.

In this paper, we consider linear MGs.

{\par \textbf{Linear (function approximation in) Markov Games.}} Under a linear MG setting, there exists a known feature map $\phi: \mathcal{S} \times \mathcal{A} \to \mathbb{R}^d$ such that for every $h\in[H]$, there exist $d$ unknown (signed) measures $\mu_{h} = \left(\mu_{h}^{(1)}, \dots \mu_{h}^{(d)}\right)$ over $\mathcal{S}$ and an unknown vector $\theta_{h} \in \mathbb{R}^d$ such that $\P_{h}(x'| x, a) = \langle \phi(x, a), \mu_{h}(x')\rangle$, $r_{h}(x, a) = \langle \phi(x, a), \theta_{h}\rangle$ for all $(x, a, x') \in \mathcal{S} \times \mathcal{A} \times \mathcal{S}$. We assume the non-scalar case with $d\geq 2$ and that the feature map satisfies $\|\phi(x, a) \|\leq 1$ for all $(x, a) \in \mathcal{S} \times \A$ and $\max\{\|\mu_{h}(\S)\|, \|\theta_{h}\|\} \leq \sqrt{d}$ at each step $h \in [H]$, where (with an abuse of notation) $\|\mu_{h}(\mathcal{S})\| = \int_{\mathcal{S}}\|\mu_{h}(x)\|dx$. Note that the transition kernel $\P_h(\cdot|x,a)$ may have infinite degrees of freedom since the measure $\mu_h$ is unknown.


{\par \textbf{Performance Metric.}} We consider that all agents are learning during a total of $K$ episodes, starting 
at some initial state $s_o \in \cS$ at the beginning of each episode. 
For a set of policies $\{\pi^{k}\}_{k\in[K]}$ provided by an online MARL algorithm, we use the following regret performance metric: 
\begin{equation}
    \label{eq:regret_nash}
    \textnormal{Regret}(K) = \sum_{k=1}^K\max_{i\in[n]} (V_1^{i,\bre(\pi^k_{-i}),\pi^k_{-i}}(s_o) - V_1^{i,\pi_i^k}(s_o)).
\end{equation}
The idea behind such regret is that, at episode $k\in[K]$, $\max_{i\in[n]} (V_1^{i,\bre(\pi^k_{-i}),\pi^k_{-i}}(s_o) - V_1^{i,\pi_i^k}(s_o))=0$ iff (product) policy $\pi^k$ is a Nash equilibrium for the Markov game.

{\par \textbf{Static games.}} We also consider that the $n$ agents can play a static game, keeping their respective action spaces. Given that each agent has an associated reward function $g_i:\A\to\R$ in a static game, the game is defined by the tuple  $(g_1,\dots,g_n)$. 
Given $a\in\A$, we define $a_{-i}$ as the respective element of $\A_{-i}=\A_1\times\dots \A_{i-1}\times\A_{i+1}\times\dots\times \A_n$. 
We consider a tuple $\nu=(\nu_1,\dots,\nu_n)$ with $\nu_i\in\Delta(\A_i)$, $i\in[n]$, to be a strategy profile; and let $\nu_{-i}$ be the tuple $\nu$ without its $i$th element. In this work, we consider strategy profiles as product measures $\nu(a)=\prod^n_{i=1}\nu(a_i)$ and so $\nu\in\Delta(\A)$; a similar consideration follows for $\nu_{-i}\in\Delta(\A_{-i})$, $i\in[n]$. A strategy profile $\nu^*$ is a Nash equilibrium if $\nu^*_i\in\argmax_{\nu\in\Delta(\A_i)}\E_{\substack{a_i\sim\nu\\a_{j}\sim\nu^*_{j}, j\in[n]\setminus\{i\}}}[g_i(a)]$ for every $i\in[n]$.

\begin{definition}[Global optimal and saddle Nash equilibria~\citep{Hu2003NashQ}]
\label{def:global-saddle}
A strategy profile $\nu^*$ of the static game $(g_1,\dots,g_n)$ is:
\begin{enumerate}[label=(\roman*)]
    \item a \emph{global optimal} (Nash) equilibrium if $\E_{a\sim\nu^*}[g_{i}(a)]\geq \E_{a\sim\nu}[g_{i}(a)]$ for any strategy profile $\nu\in\Delta(\A)$; and 
    \item a \emph{saddle} Nash equilibrium if $\E_{a\sim\nu^*}[g_{i}(a)]\geq \E_{\substack{a_i\sim\nu_i\\a_{-i}\sim\nu^*_{-i}}}[g_{i}(a)]$ for any $i\in[n]$ and any $\nu_i\in\Delta(\A_i)$, and $\E_{a\sim\nu^*}[g_{i}(a)]\leq \E_{\substack{a_i\sim\nu^*_i\\a_{-i}\sim\nu_{-i}}}[g_{i}(a)]$ for any strategy profile $\nu_{-i}\in\Delta(\A_{-i})$.
\end{enumerate}   
\end{definition}


\section{Nash Q-learning and its analysis}
\label{sec:NashQ-analysis}

We propose a simple Nash Q-learning algorithm based on linear function approximation and optimism named 
\emph{\algname} or \algbrev 
~as described in Algorithm~\ref{alg:main}.

We provide an outline of the \algbrev ~algorithm. At each iteration $k\in[K]$ and step $h\in[H]$, 
the information of the explored state-action trajectories described by 
the agents in the game at the same step but up to the previous episode is collected in a covariance matrix $\Lambda^k_h$ (line 6 of Algorithm~\ref{alg:main}). Then, all the agents participate in a static game described by some prior optimistic estimates of Q-functions 
and a Nash equilibrium is computed 
-- since this static game is solved in every episode and time-step (and depends on the current state of the Markov game), it is called a \emph{stage game}. Then each agent, using its computed Nash policy from the stage game, computes a new \emph{optimistic} estimate of the Q-function (line 10), using the optimism bonus $\beta(\phi(\cdot,\cdot)^\top(\Lambda^k_h)^{-1}\phi(\cdot,\cdot))^{1/2}$. Then, all the agents jointly explore the environment (lines 14-16) by taking actions coming from their respective policies computed from stage games. The resulting state-action trajectory across the episodes will then be collected  
and the whole process repeats. 

\begin{remark}[About 
information access in \algbrev]
In this paper, we are primarily concerned with analyzing \algbrev ~as a solver for the policies of the underlying Markov game,  e.g, as done in the recent work~\citep{Liu2021SharpSelfPlay}. One could think of relaxing some implementation details such as the information each agent has access to across episodes, but this is beyond the scope of the paper. For example, at each step $h\in[H]$ and iteration $k\in[K]$, 
one could make the optimistic Q-functions of $i\in[n]$, $Q_h^{i,k}$, be private information to the rest of the agents. Then, each agent would try to estimate the Q-functions of the rest of the agents based on the observation of 
past rewards -- an idea already outlined in~\citep{Hu2003NashQ}. 
\end{remark}

\begin{algorithm}[t!]
  \caption{\algname~ (\algbrev)}
  \label{alg:main}
\begin{algorithmic}[1]
    \STATE {\bfseries Input:} $K$, $\beta$, $\lambda$
    \FOR{episode $k\in[K]$}
        \STATE $x_1^{k}\gets s_0$ 
         \STATE $Q_{H+1}^{i,k}(\cdot,\cdot)\gets 0$, $i\in[n]$
        \FOR{$h=H,\dots,1$}
            \STATE $\Lambda_{h}^k\gets\lambda I_d + \sum^{k-1}_{\tau=1}\phi(x_h^{\tau},a_h^{\tau})\phi(x_h^{\tau},a_h^{\tau})^\top$
            \STATE $\pi^*\gets$ a  Nash Equilibrium for the $n$-player game $(Q^{1,k}_{h+1}(x^k_{h+1},\cdot),\dots,Q^{n,k}_{h+1}(x^k_{h+1},\cdot))$
            \FOR{$i\in[n]$}
            \STATE $w_{h}^{i,k}\gets(\Lambda_{h}^{k})^{-1} \sum^{k-1}_{\tau=1}\phi(x_h^{\tau},a_h^{\tau})$ $[r_h^i(x_h^\tau,a_h^\tau)+\E_{a\sim\pi^*}[Q^{i,k}_{h+1}(x_{h+1}^{\tau},a)]$
            \STATE $Q^{i,k}_{h}(\cdot,\cdot)\gets \min\{(w_{h}^{i,k})^\top\phi(\cdot,\cdot)+\beta(\phi(\cdot,\cdot)^\top(\Lambda_{h}^{k})^{-1}\phi(\cdot,\cdot))^{1/2},H\}$
            \ENDFOR 
        \ENDFOR
        \FOR{$h\in[H]$}
            \STATE $\pi^k_h(x^k_h)\gets$ a Nash Equilibrium for the $n$-player game $(Q^{1,k}_{h}(x^k_{h},\cdot),\dots,Q^{n,k}_{h}(x^k_{h},\cdot))$
            \STATE Take $a_h^k \sim \pi^k_h(x^k_h)$
            \STATE Observe $x_{h+1}^{k}$ 
        \ENDFOR
    \ENDFOR
\end{algorithmic}
\end{algorithm}

We now present the paper's main result.

\begin{theorem}[Performance of the \algbrev ~algorithm]
\label{thm:main-nashQ}
There exists an absolute constant $c_\beta>0$ such that, for any fixed $\delta\in(0,1)$, if we set $\lambda=1$ and $\beta=c_\beta dH\sqrt{\iota}$, with $\iota:=\log(dKH/\delta)$, then, with probability at least $1-(n+2)\delta$, 
\begin{equation}
\label{eqn:regret-res2}
\textnormal{Regret}(K)
\leq 
\cO\bigg(\sqrt{K}\sqrt{d^3H^5\iota^2}\bigg).
\end{equation}
\end{theorem}

{\par \textbf{Sample efficiency.}} Our regret bound is sublinear in the number of episodes $K$ and -- ignoring logarithmic terms -- polynomial on the parameters $d$ and $H$, 
i.e., there is learning with sample efficiency. Our finite-sample guarantee states that $K=\tilde{\cO}\left(\frac{d^3H^5}{\epsilon^2}\right)$ episodes are needed in order to achieve an average regret less or equal than $\epsilon$, i.e., for the policies across the episodes to perform on average as an $\epsilon$-Nash equilibrium. 

{\par \textbf{About the number of agents.}} 
Our bound has a logarithmic dependence on the number of agents $n$. However, we remark that the feature dimension $d$ of the linear MG \emph{might} hide dependencies on $n$ depending on how the feature vector $\phi$ is constructed (more on this below, when discussing the tabular case).
In any case, the larger the number of agents, the more samples are needed to achieve the same average regret performance.
Intuitively, this makes sense, since increasing the number of agents increases the number of possible decision makers and thus the complexity of the state-action space to be sampled. This is in stark contrast with other works in the single-agent RL case where multiple agents can be deployed to explore the \emph{same} state-action space of the MDP, in which case their performance measure improves with the number of agents~\citep{Cisneros-Velarde2022OnePolicyEnough}.

{\par \textbf{Comparison with (single-agent) RL.}} For the classic single-agent RL case ($n=1$), \cite{CJ-ZY-ZW-MIJ:20}  obtained, with the regret metric with respect to the optimal policy of the underlying MDP, the bound $\tilde{\cO}(\sqrt{K}\sqrt{d^3 H^4})$. Thus, our result is larger 
by a factor $H$ -- thus nearly-matching the sample efficiency. Having to learn a Nash equilibrium of an MG thus requires more samples than what would be necessary for an MDP. It is important to highlight that though the single-agent case requires taking an action that maximizes the optimistic Q-function (see~\citep[LSVI-UCB 
Algorithm]{CJ-ZY-ZW-MIJ:20}), \algbrev ~requires solving for Nash equilibrium and thus is computationally more complex. 

{\par \textbf{Comparison with~\citep{Hu2003NashQ}.}} The original Nash Q-learning proposed by~\cite{Hu2003NashQ} has as its performance metric the convergence to a Nash equilibrium of the underlying discounted MG. In order to ensure such convergence, they assumed the existence of either global optimal or saddle Nash equilibria uniformly on every stage game -- see Definition~\ref{def:global-saddle}. In contrast, since we use regret in the context of episodic MGs, we are interested in the average performance of the computed policies across iterations, with the expectation that it will approximate a Nash equilibrium performance. Therefore, we are not strictly interested in convergence to a \emph{single} Nash equilibrium. For this reason, our proof makes no use of the assumptions across stage games by~\cite{Hu2003NashQ}. Their work and ours, though being model-free, use completely different proof techniques.

{\par \textbf{Comparison with~\citep{Liu2021SharpSelfPlay}.}} The first Nash Q-learning algorithm in~\citep{Hu2003NashQ} was designed and analyzed for tabular RL. Motivated by concerns of large or continuous state spaces, we decided to opt for the function approximation regime. As it is known in the literature, a direct translation of the \algbrev ~algorithm to the tabular case can be done by letting the feature vector $\phi$ capture $d=|\cS||\A|=|\cS|\prod^n_{i=1}|\A_i|$, which would give our regret bound a complexity of $\tilde{\cO}(\sqrt{H^5|\S|^3|(\prod^n_{i=1}|\A_i|)^3K})$. In the tabular case,~\cite{Liu2021SharpSelfPlay} proposed the \emph{Multi-Nash-VI} algorithm which obtains $\tilde{\cO}(\sqrt{H^4|\S|^2|(\prod^n_{i=1}|\A_i|)K})$ -- tighter in horizon $H$ and both sizes of the state and action spaces of the agents. Interestingly, in the tabular case, both \algbrev ~and Multi-Nash-VI 
are of different nature, 
since the former is model-free and the latter model-based. Interestingly as well, Multi-Nash-VI requires the computation of two Q-functions based on the constructed model -- one using optimism and another using pessimism --, whereas \algbrev ~requires only the computation of an optimistic Q-function. 
As generally expected in general-sum MGs,  
both suffer from the
\emph{curse of multi-agents} in the tabular case
since the sample bounds have exponential dependence on the number of agents (through the product of the cardinality of the agents' action spaces)~\citep{song2022whenlearningGenSum}.

\section{Proving the main result}
\label{sec:NashQproof}

We first present two lemmas that make use of the fact that we solve for Nash equilibria in the stage games. All missing proofs and results are in the Appendix.

\begin{lemma}[Bounding the covering number]\label{lem:bound-import}
Let $i\in[n]$, and let $\bar{w}_i\in\R^d$ be such that $\norm{\bar{w}_i}\leq L$, $\bar{\Lambda}\in\R^{d\times d}$ be such that its minimum eigenvalue is greater or equal than $\lambda$, and, for all $(x,a)\in\S\times\A$, let $\phi(x,a)\in\R^d$ be such that $\norm{\phi(x, a)}\leq 1$, and let $\beta>0$.
Define the function class 
\begin{equation}
\label{eq:function_class}
\mathcal{V}_i = \Big\{V:\S\to\R\;\Big|\; V(\cdot)
=\max_{\nu\in\Delta(\A_i)}\E_{\substack{a_i\sim\nu\\a_{-i}\sim\pi_{-i}(\cdot)}}\Big[\min\Big\{ \bar{w}_i^\top\phi(\cdot, a)
+ \beta \sqrt{\phi(\cdot, a)^\top\bar{\Lambda}^{-1} \phi(\cdot, a)}, H \Big\}\Big]\Big\},
\end{equation}
where 
$\pi_{-i}(\cdot)\in\Delta(\A_{-i})$. Let $\mathcal{N}_{\epsilon}$ be the $\epsilon$-covering number of $\mathcal{V}_i$ with respect to the distance $\operatorname{dist}(V, V') = \sup_{x\in\S} |V(x) - V'(x)|$. Then, 
$$
\log \mathcal{N}_\epsilon \leq d  \log (1+ 4L/ \epsilon ) + d^2 \log \bigl [ 1 +  8 d^{1/2} \beta^2  / (\lambda\epsilon^2)  \bigr ].
$$
\end{lemma}

\begin{lemma}[Optimism bounds]
\label{lem:optimism_bound} Consider the setting of Theorem~\ref{thm:main-nashQ}. Given the event $\mathcal{E}_i$ defined in Lemma~B.2, we have for all $(x, a, h, k) \in \S\times\A\times[H]\times[K]$ that  \begin{equation*}
Q^{i,\bre(\pi_{-i}^k),\pi_{-i}^k}_h(x,a)\leq Q^{i,k}_h(x,a) \text{ and }\quad V^{i,\bre(\pi_{-i}^k),\pi_{-i}^k}_h(x)\leq V^{i,k}_h(x).
\end{equation*}
\end{lemma}

{\par \textbf{The importance of Lemma~\ref{lem:bound-import} and Lemma~\ref{lem:optimism_bound}}} Lemma~\ref{lem:bound-import} defines a function class $\mathcal{V}_i$ to which the function $V^{i,k}_h(\cdot)=\E_{a\sim\pi^k_h(\cdot)}[Q^{i,k}_h(\cdot,a)]$ belongs -- line 14 from Algorithm~\ref{alg:main}. Indeed, the characterization of $\mathcal{V}_i$ includes the one of a Nash equilibrium for a static game; however, we remark that $\pi_{-i}(\cdot)\in\Delta(\A_{-i})$ in the statement of Lemma~\ref{lem:bound-import} does not need to be a product measure. Using a covering number argument, Lemma~\ref{lem:bound-import} would be used to prove a series of results that would end up being used by Lemma~\ref{lem:optimism_bound}. Fundamentally, Lemma~\ref{lem:optimism_bound} makes use of (i) the optimism bonus at each episode -- the factor starting with $\beta$ in line 10 of \algbrev ~-- and (ii) the selection of Nash equilibria across all stage games. Bounding the best-response value functions across agents in Lemma~\ref{lem:optimism_bound} is important because it upper bounds one of the terms of the regret, see~\eqref{eq:regret_nash}. 
Finally, we end our discussion by pointing out that these lemmas are the only two places in the proof of Theorem~\ref{thm:main-nashQ} which makes direct use of the notion of Nash equilibria.


\subsection{Proof sketch
of Theorem~\ref{thm:main-nashQ}}


We present the proof sketch of our main result. A more detailed full version of the proof along with all auxiliary results and necessary proofs are found in the Appendix.  


Let us first condition on the event $\bigcap_{i=1}^n\mathcal{E}_i$ where $\mathcal{E}_i$ is defined in Lemma~B.2.
Since $\Pe[\text{not }  \mathcal{E}_i]\leq \delta$, applying union bound let us conclude that $\Pe[\bigcap_{i\in[n]}\mathcal{E}_i]\geq 1-n\delta
$.
Conditioning on this event allows us to use Lemma~\ref{lem:optimism_bound} for every $i\in[n]$.

For any $k\in[K]$, given the policy $\pi^k=\{\pi^k_i\}_{i\in[n]}$ defined by \algbrev, we define the functions $\hat{Q}_h^k$ 
and $\hat{V}_h^k$ 
recursively as: $\hat{V}_{H+1}^k(x)=\hat{Q}_{H+1}^k(x)=0$ and 
\begin{align*}
\hat{Q}_h^k(x,a) &= \Pe_h\hat{V}^k_{h+1}(x,a)+2\beta\sqrt{(\phi^k_h)^\top(\Lambda^k_h)^{-1}\phi^k_h},\\
\hat{V}_h^k(x)&=\E_{a\sim\pi^k_h(x)}[\hat{Q}_h^k(x,a)]
\end{align*}
for any $h=H,\dots,1$ and $(x,a)\in\S\times\A$. Notice that since $2\beta\sqrt{(\phi^k_h)^\top(\Lambda^k_h)^{-1}\phi^k_h}\leq 2\beta\sqrt{(\phi^k_h)^\top\phi^k_h}=2\beta\norm{\phi^k_h}\leq 2\beta$, we have that $\hat{Q}_h^k$ 
and $\hat{V}_h^k$ are nonnegative with maximum value $2\beta H$.

Let $k\in[K]$. We 
can show that
for any $(h,x,a)\in[H]\times\S\times\A$, 
\begin{equation}
\label{eq:claim-upper-bound}
\begin{aligned}
        \max_{i\in[n]}(Q^{i,k}_h(x,a)-Q^{i,\pi^k}_h(x,a))&\leq \hat{Q}_h^k(x,a)\text{, and}\\
        \max_{i\in[n]}(V^{i,k}_h(x)-V^{i,\pi^k}_h(x))&\leq \hat{V}_h^k(x).
\end{aligned}
\end{equation}

We now introduce the following notation: $\delta^{k}_h := \E_{a\sim\pi^k_h}[\hat{Q}^k_h(x^k_h,a)]-\hat{Q}^k_h(x^k_h,a^k_h)$, and $\xi^k_{h+1} := 
\Pe_h\hat{V}^k_{h+1}(x^k_h,a^k_h) - \hat{V}^k_{h+1}(x^k_{h+1})$ with $\xi^k_1:=0$. Then, for any $(h,k) \in [H] \times [K]$, we can show that 
\begin{equation*}
\hat{V}^k_h(x^k_h)=\delta^k_h+\xi^k_{h+1}+2\beta\sqrt{(\phi^k_h)^\top(\Lambda^k_h)^{-1}\phi^k_h}+\hat{V}^k_{h+1}(x^k_{h+1}).
\end{equation*}

Now, let us focus on the regret performance metric.
\begin{equation}
\label{eq:regret_prev_p}
\begin{aligned}
\textnormal{Regret}(K) &=  \sum_{k=1}^K\max_{i\in[n]}(V_1^{i,\bre(\pi^k_{-i}),\pi^k_{-i}}(s_o) - V_1^{i,\pi^k}(s_o))\\
&\overset{(a)}{\leq}\sum_{k=1}^{K} \max_{i\in[n]}(V^{i,k}_1(s_o) - V^{i,\pi^k}_1 (s_o))\\
&\overset{(b)}{\leq}\sum_{k=1}^{K} \hat{V}^{k}_1(s_o)
\\
&=\underbrace{\sum_{k=1}^K\sum_{h=1}^H\xi^k_h}_{\textrm{(I)}}+\underbrace{\sum_{k=1}^{K}\sum_{h=1}^H \delta^{k}_{h}}_{\textrm{(II)}} + \underbrace{2\beta \sum_{k=1}^{K}\sum_{h=1}^H \sqrt{(\phi^{k}_h)^\top (\Lambda^k_h)^{-1}\phi^{k}_h}}_{\textrm{(III)}},
\end{aligned}
\end{equation}
where (a) follows from Lemma~\ref{lem:optimism_bound} and the fact that we are conditioned on the event $\bigcap^n_{i=1}\mathcal{E}_i$; (b) follows from~\eqref{eq:claim-upper-bound}.

We first analyze the term (I) from~\eqref{eq:regret_prev_p}. 
By defining an appropriate infinite sequence of tuples $\L^\star\subset \mathbb{Z}_{\geq 1}\times[H]$, we can show that $\{\xi^k_h\}_{(k,h)\in\L^\star}$ is a martingale difference sequence.  
Therefore, we can use the Azuma-Hoeffding inequality to conclude that, for any $\epsilon > 0$,
\begin{equation*}
\Pr \left(\sum_{k=1}^{K}\sum_{h=1}^H  \xi^{k}_{h}> \epsilon \right) \leq \exp \bigg (\frac{-2 \epsilon^2 } {(KH)(16\beta^2H^2) } \bigg ).
\end{equation*}
We choose $\epsilon=\sqrt{8KH^3\beta^2\log\left(\frac{1}{\delta}\right)}$. Then, with probability at least $1 -\delta$,   
\begin{equation}
\label{eq:final2}
  \textrm{(I)}=\sum_{k=1}^{K}\sum_{h=1}^H  \xi^k_{h}\leq \sqrt{8KH^3\beta^2\log\left(\frac{1}{\delta}\right)} \leq 8\beta H\sqrt{KH\iota}, 
\end{equation}
recalling that $\iota = \log\left(\frac{dKH}{\delta}\right)$. We call $\bar{\mathcal{E}}$ the event such that~\eqref{eq:final2} holds.

The term (II) can be analyzed in a very similar way as in (I) to show that $\{\delta^k_h\}_{(k,h)\in\cL^*}$ is a martingale difference sequence, and thus obtain that with probability at least $1-\delta$,
\begin{equation}
\label{eq:final3}
  \textrm{(II)}=\sum_{k=1}^{K}\sum_{h=1}^H  \delta^k_{h}\leq 8\beta H\sqrt{KH\iota}. 
\end{equation}
We call $\tilde{\mathcal{E}}$ the event such that~\eqref{eq:final3} holds.

We now analyze the term (III) from~\eqref{eq:regret_prev_p}.
\begin{equation}
\label{eq:aux_last_2}
\textrm{(III)}=2\beta \sum_{h=1}^H\sum_{k=1}^{K} \sqrt{(\phi^{k}_h)^\top (\Lambda^k_h)^{-1}\phi^{k}_h}
  \overset{\textrm{(a)}}{\leq}2\beta \sum_{h=1}^H  \sqrt{K}\sqrt{ \sum_{k=1}^K (\phi^{k}_h)^\top (\Lambda^k_h)^{-1}\phi^{k}_h}
\overset{(b)}{\leq} 2\beta H \sqrt{2dK\iota}, 
\end{equation}
where (a) follows from the Cauchy-Schwartz inequality, and we can show (b) 
by using the so-called elliptical potential lemma~\citep[Lemma~11]{YAY-DP-CS:11}.

Now, using the results in~\eqref{eq:final2}, \eqref{eq:final3}, and~\eqref{eq:aux_last_2} back in~\eqref{eq:regret_prev_p}, we conclude that,
\begin{equation}
\label{eq:regret_almost_l}
\begin{aligned}
\textnormal{Regret}(K) &\leq  
8\beta H\sqrt{KH\iota} + 8\beta H\sqrt{KH\iota}
+ 
2\beta H \sqrt{dK\iota}\\
&=16c_\beta\sqrt{d^2 K H^5\iota^2}
+ 
2c_\beta\sqrt{ d^3KH^4\iota^2}
\\
&\overset{\textrm{(a)}}{\leq} 18c_\beta\sqrt{d^3KH^5\iota^2},
\end{aligned}
\end{equation}
where (a) follows from $\sqrt{\iota}\leq \iota$.

Finally, applying union bound let us conclude that 
$\Pe[\bigcap_{i\in[n]}\mathcal{E}_i\cap\bar{\mathcal{E}}\cap\tilde{\mathcal{E}}]\geq 1-(n+2)\delta
$,
i.e., our final result holds with probability at least $1-(n+2)\delta
$. This finishes the proof of Theorem~\ref{thm:main-nashQ}.

\section{Conclusion}
\label{sec:conclusion}
We have shown the sample-efficiency of Nash Q-learning under linear function approximation -- ideal for large state spaces or continuous ones -- by making use of the principle of optimism in the face of uncertainty -- largely exploited in the modern RL literature. We also compared our result to the sample complexity obtained for single-agent RL with linear function approximation and for general-sum MARL on the tabular case. We hope our work may open the path to the future analysis of a more diverse set of MARL algorithms.

One future research direction is 
obtaining sample performance lower bounds to analyze the (closeness to) minimax optimality of general-sum MARL algorithms such as Nash Q-learning. 
Moreover, though most modern theoretical work in RL (including this paper) mostly focus on sample efficiency, it is relevant to propose and study algorithms that are also computational efficient -- for which other weaker solutions to MGs are useful. Finally, another future direction would be to expand the analysis of Nash Q-learning to 
nonlinear function approximators
such as neural networks.
%


\bibliography{NashQ}

\begin{thebibliography}{36}
\providecommand{\natexlab}[1]{#1}
\providecommand{\url}[1]{\texttt{#1}}
\expandafter\ifx\csname urlstyle\endcsname\relax
  \providecommand{\doi}[1]{doi: #1}\else
  \providecommand{\doi}{doi: \begingroup \urlstyle{rm}\Url}\fi

\bibitem[Abbasi-yadkori et~al.(2011)Abbasi-yadkori, P\'{a}l, and
  Szepesv\'{a}ri]{YAY-DP-CS:11}
Yasin Abbasi-yadkori, D\'{a}vid P\'{a}l, and Csaba Szepesv\'{a}ri.
\newblock Improved algorithms for linear stochastic bandits.
\newblock In \emph{Advances in Neural Information Processing Systems},
  volume~24. Curran Associates, Inc., 2011.

\bibitem[Agarwal et~al.(2020)Agarwal, Kakade, Krishnamurthy, and
  Sun]{agarwal2020flambe}
Alekh Agarwal, Sham Kakade, Akshay Krishnamurthy, and Wen Sun.
\newblock Flambe: Structural complexity and representation learning of low rank
  mdps.
\newblock \emph{Advances in neural information processing systems},
  33:\penalty0 20095--20107, 2020.

\bibitem[Bai et~al.(2020)Bai, Jin, and Yu]{Bai2020SelfPlay}
Yu~Bai, Chi Jin, and Tiancheng Yu.
\newblock Near-optimal reinforcement learning with self-play.
\newblock In \emph{Proceedings of the 34th International Conference on Neural
  Information Processing Systems}, 2020.

\bibitem[Chen et~al.(2022)Chen, Zhou, and Gu]{chen2022almost}
Zixiang Chen, Dongruo Zhou, and Quanquan Gu.
\newblock Almost optimal algorithms for two-player zero-sum linear mixture
  {Markov} games.
\newblock In \emph{International Conference on Algorithmic Learning Theory},
  pages 227--261. PMLR, 2022.

\bibitem[Cisneros-Velarde et~al.(2022)Cisneros-Velarde, Lyu, Koyejo, and
  Kolar]{Cisneros-Velarde2022OnePolicyEnough}
Pedro Cisneros-Velarde, Boxiang Lyu, Sanmi Koyejo, and Mladen Kolar.
\newblock One policy is enough: Parallel exploration with a single policy is
  near-optimal for reward-free reinforcement learning.
\newblock \emph{arXiv preprint arXiv:2205.15891}, 2022.

\bibitem[Daskalakis et~al.(2009)Daskalakis, Goldberg, and
  Papadimitriou]{Daskalakis2019Complexity}
Constantinos Daskalakis, Paul~W. Goldberg, and Christos~H. Papadimitriou.
\newblock The complexity of computing a nash equilibrium.
\newblock \emph{SIAM Journal on Computing}, 39\penalty0 (1):\penalty0 195--259,
  2009.
\newblock \doi{10.1137/070699652}.

\bibitem[Dinneweth et~al.(2022)Dinneweth, Boubezoul, Mandiau, and
  Espié]{Dinneweth2022SurveyCar}
Joris Dinneweth, Abderrahmane Boubezoul, René Mandiau, and Stéphane Espié.
\newblock Multi-agent reinforcement learning for autonomous vehicles: a survey.
\newblock \emph{Autonomous Intelligent Systems}, 2\penalty0 (27), 2022.

\bibitem[Dubey and Pentland(2021)]{dubey2021provably}
Abhimanyu Dubey and Alex Pentland.
\newblock Provably efficient cooperative multi-agent reinforcement learning
  with function approximation.
\newblock \emph{arXiv preprint arXiv:2103.04972}, 2021.

\bibitem[Foruzan et~al.(2018)Foruzan, Soh, and Asgarpoor]{foruzan18Microgrid}
Elham Foruzan, Leen-Kiat Soh, and Sohrab Asgarpoor.
\newblock Reinforcement learning approach for optimal distributed energy
  management in a microgrid.
\newblock \emph{IEEE Transactions on Power Systems}, 33\penalty0 (5):\penalty0
  5749--5758, 2018.
\newblock \doi{10.1109/TPWRS.2018.2823641}.

\bibitem[Gao et~al.(2021)Gao, Xie, Du, and Yang]{MG-RX-SSD-LFY:21}
Minbo Gao, Tianle Xie, Simon~S. Du, and Lin~F. Yang.
\newblock A provably efficient algorithm for linear {Markov} decision process
  with low switching cost.
\newblock \emph{arXiv preprint arXiv:2101.00494}, 2021.

\bibitem[Hernandez-Leal et~al.(2019)Hernandez-Leal, Kaisers, Baarslag, and
  de~Cote]{HernandezLeal2019Survey}
Pablo Hernandez-Leal, Michael Kaisers, Tim Baarslag, and Enrique~Munoz de~Cote.
\newblock A survey of learning in multiagent environments: Dealing with
  non-stationarity.
\newblock \emph{arXiv preprint arXiv:1707.09183}, 2019.

\bibitem[Hu and Wellman(2003)]{Hu2003NashQ}
Junling Hu and Michael~P. Wellman.
\newblock Nash q-learning for general-sum stochastic games.
\newblock \emph{Journal of Machine Learning Research}, \penalty0 (4):\penalty0
  1039--1069, 2003.

\bibitem[Huang et~al.(2021)Huang, Chen, Zhao, Qin, Jiang, and
  Liu]{huang2021deployment}
Jiawei Huang, Jinglin Chen, Li~Zhao, Tao Qin, Nan Jiang, and Tie-Yan Liu.
\newblock Towards deployment-efficient reinforcement learning: Lower bound and
  optimality.
\newblock In \emph{International Conference on Learning Representations}, 2021.

\bibitem[Jin et~al.(2020)Jin, Yang, Wang, and Jordan]{CJ-ZY-ZW-MIJ:20}
Chi Jin, Zhuoran Yang, Zhaoran Wang, and Michael~I Jordan.
\newblock Provably efficient reinforcement learning with linear function
  approximation.
\newblock In \emph{Proceedings of Thirty Third Conference on Learning Theory},
  volume 125 of \emph{Proceedings of Machine Learning Research}, pages
  2137--2143. PMLR, 09--12 Jul 2020.

\bibitem[Jin et~al.(2022)Jin, Liu, Wang, and Yu]{jin2022vlearning}
Chi Jin, Qinghua Liu, Yuanhao Wang, and Tiancheng Yu.
\newblock V-learning -- a simple, efficient, decentralized algorithm for
  multiagent {RL}.
\newblock In \emph{ICLR 2022 Workshop on Gamification and Multiagent
  Solutions}, 2022.

\bibitem[Jin et~al.(2021)Jin, Yang, and Wang]{Jin2021Pevi}
Ying Jin, Zhuoran Yang, and Zhaoran Wang.
\newblock Is pessimism provably efficient for offline rl?
\newblock In \emph{International Conference on Machine Learning}, 2021.

\bibitem[Kozuno et~al.(2021)Kozuno, M{\'e}nard, Munos, and
  Valko]{kozuno2021learning}
Tadashi Kozuno, Pierre M{\'e}nard, Remi Munos, and Michal Valko.
\newblock Learning in two-player zero-sum partially observable {Markov} games
  with perfect recall.
\newblock \emph{Advances in Neural Information Processing Systems}, 34, 2021.

\bibitem[Li et~al.(2022)Li, Chi, Wei, and Chen]{li2022minimaxoptimal}
Gen Li, Yuejie Chi, Yuting Wei, and Yuxin Chen.
\newblock Minimax-optimal multi-agent {RL} in markov games with a generative
  model.
\newblock In \emph{Advances in Neural Information Processing Systems}, 2022.

\bibitem[Liu et~al.(2021)Liu, Yu, Bai, and Jin]{Liu2021SharpSelfPlay}
Qinghua Liu, Tiancheng Yu, Yu~Bai, and Chi Jin.
\newblock A sharp analysis of model-based reinforcement learning with
  self-play.
\newblock In \emph{Proceedings of the 38th International Conference on Machine
  Learning}. PMLR, 2021.

\bibitem[Mao and Başar(2022)]{Mao202REfficientRLGeneralSum}
Weichao Mao and Tamer Başar.
\newblock Provably efficient reinforcement learning in decentralized
  general-sum markov games.
\newblock \emph{Dynamic Games and Applications}, 2022.

\bibitem[Ni et~al.(2022)Ni, Song, Zhang, Jin, and Wang]{chengzhuo2022RLGSMG}
Chengzhuo Ni, Yuda Song, Xuezhou Zhang, Chi Jin, and Mengdi Wang.
\newblock Representation learning for general-sum low-rank markov games.
\newblock \emph{arXiv preprint arXiv:2210.16976}, 2022.

\bibitem[Ni et~al.(2014)Ni, Liu, Ren, and Yang]{Ni2014WaterResource}
Jianjun Ni, Minghua Liu, Li~Ren, and Simon~X. Yang.
\newblock A multiagent q-learning-based optimal allocation approach for urban
  water resource management system.
\newblock \emph{IEEE Transactions on Automation Science and Engineering},
  11\penalty0 (1):\penalty0 204--214, 2014.
\newblock \doi{10.1109/TASE.2012.2229978}.

\bibitem[Pham et~al.(2018)Pham, La, Feil-Seifer, and
  Nefian]{pham2018cooperative}
Huy~Xuan Pham, Hung~Manh La, David Feil-Seifer, and Aria Nefian.
\newblock Cooperative and distributed reinforcement learning of drones for
  field coverage.
\newblock \emph{arXiv preprint arXiv:1803.07250}, 2018.

\bibitem[Qiu et~al.(2021)Qiu, Ye, Wang, and Yang]{SQ-JY-ZW-ZY:22}
Shuang Qiu, Jieping Ye, Zhaoran Wang, and Zhuoran Yang.
\newblock On reward-free {RL} with kernel and neural function approximations:
  Single-agent {MDP} and {Markov} game.
\newblock In \emph{International Conference on Machine Learning}, 2021.

\bibitem[Silver et~al.(2016)Silver, Huang, Maddison, Guez, Sifre, van~den
  Driessche, Schrittwieser, Antonoglou, Panneershelvam, Lanctot, Dieleman,
  Grewe, Nham, Kalchbrenner, Sutskever, Lillicrap, Leach, Kavukcuoglu, Graepel,
  and Hassabis]{alphaGo:16}
David Silver, Aja Huang, Chris~J. Maddison, Arthur Guez, Laurent Sifre, George
  van~den Driessche, Julian Schrittwieser, Ioannis Antonoglou, Veda
  Panneershelvam, Marc Lanctot, Sander Dieleman, Dominik Grewe, John Nham, Nal
  Kalchbrenner, Ilya Sutskever, Timothy Lillicrap, Madeleine Leach, Koray
  Kavukcuoglu, Thore Graepel, and Demis Hassabis.
\newblock Mastering the game of {Go} with deep neural networks and tree search.
\newblock \emph{Nature}, 2016.

\bibitem[Silver et~al.(2017)Silver, Schrittwieser, Simonyan, Antonoglou, Huang,
  Guez, Hubert, Baker, Lai, Bolton, Chen, Lillicrap, Hui, Sifre, van~den
  Driessche, Graepel, and Hassabis]{alphaZero:17}
David Silver, Julian Schrittwieser, Karen Simonyan, Ioannis Antonoglou, Aja
  Huang, Arthur Guez, Thomas Hubert, Lucas Baker, Matthew Lai, Adrian Bolton,
  Yutian Chen, Timothy Lillicrap, Fan Hui, Laurent Sifre, George van~den
  Driessche, Thore Graepel, and Demis Hassabis.
\newblock Mastering the game of {Go} without human knowledge.
\newblock \emph{Nature}, 2017.

\bibitem[Song et~al.(2022)Song, Mei, and Bai]{song2022whenlearningGenSum}
Ziang Song, Song Mei, and Yu~Bai.
\newblock When can we learn general-sum markov games with a large number of
  players sample-efficiently?
\newblock In \emph{International Conference on Learning Representations}, 2022.

\bibitem[Vinyals et~al.(2019)Vinyals, Babuschkin, Czarnecki, Mathieu, Dudzik,
  Chung, Choi, Powell, Ewalds, Georgiev, et~al.]{vinyals2019grandmaster}
Oriol Vinyals, Igor Babuschkin, Wojciech~M Czarnecki, Micha{\"e}l Mathieu,
  Andrew Dudzik, Junyoung Chung, David~H Choi, Richard Powell, Timo Ewalds,
  Petko Georgiev, et~al.
\newblock Grandmaster level in {StarCraft} {II} using multi-agent reinforcement
  learning.
\newblock \emph{Nature}, 575\penalty0 (7782):\penalty0 350--354, 2019.

\bibitem[Wagenmaker et~al.(2022)Wagenmaker, Chen, Simchowitz, Du, and
  Jamieson]{AW-YC-MS-SSD-KJ:22}
Andrew Wagenmaker, Yifang Chen, Max Simchowitz, Simon~S Du, and Kevin Jamieson.
\newblock Reward-free {RL} is no harder than reward-aware {RL} in linear
  {Markov} decision processes.
\newblock In \emph{International Conference on Machine Learning}, pages
  22430--22456. PMLR, 2022.

\bibitem[Wang et~al.(2020)Wang, Du, Yang, and Salakhutdinov]{RW-SD-LY-RS:20}
Ruosong Wang, Simon~S Du, Lin Yang, and Russ~R Salakhutdinov.
\newblock On reward-free reinforcement learning with linear function
  approximation.
\newblock In \emph{Advances in Neural Information Processing Systems},
  volume~33, pages 17816--17826, 2020.

\bibitem[Yang and Wang(2020)]{Yang2020RL}
Lin Yang and Mengdi Wang.
\newblock Reinforcement learning in feature space: Matrix bandit, kernels, and
  regret bound.
\newblock In \emph{Proceedings of the 37th International Conference on Machine
  Learning}, volume 119 of \emph{Proceedings of Machine Learning Research},
  pages 10746--10756. PMLR, 13--18 Jul 2020.

\bibitem[Yang et~al.(2020)Yang, Sun, Ma, and Wei]{Lingxiao2020Energy}
Lingxiao Yang, Qiuye Sun, Dazhong Ma, and Qinglai Wei.
\newblock Nash q-learning based equilibrium transfer for integrated energy
  management game with we-energy.
\newblock \emph{Neurocomputing}, 396:\penalty0 216--223, 2020.
\newblock ISSN 0925-2312.
\newblock \doi{https://doi.org/10.1016/j.neucom.2019.01.109}.

\bibitem[Yang and Wang(2021)]{Yang2021Survey}
Yaodong Yang and Jun Wang.
\newblock An overview of multi-agent reinforcement learning from game
  theoretical perspective.
\newblock \emph{arXiv preprint arXiv:2011.00583}, 2021.

\bibitem[Zhang et~al.(2020)Zhang, Kakade, Basar, and Yang]{zhang2020model}
Kaiqing Zhang, Sham Kakade, Tamer Basar, and Lin Yang.
\newblock Model-based multi-agent {RL} in zero-sum {Markov} games with
  near-optimal sample complexity.
\newblock \emph{Advances in Neural Information Processing Systems},
  33:\penalty0 1166--1178, 2020.

\bibitem[Zhang et~al.(2021)Zhang, Yang, and Ba{\c{s}}ar]{zhang2021multi}
Kaiqing Zhang, Zhuoran Yang, and Tamer Ba{\c{s}}ar.
\newblock Multi-agent reinforcement learning: A selective overview of theories
  and algorithms.
\newblock \emph{Handbook of Reinforcement Learning and Control}, pages
  321--384, 2021.

\bibitem[Zinkevich et~al.(2005)Zinkevich, Greenwald, and
  Littman]{Zinkevich2005CyclicEquilibria}
Martin Zinkevich, Amy Greenwald, and Michael Littman.
\newblock Cyclic equilibria in markov games.
\newblock In \emph{Advances in Neural Information Processing Systems},
  volume~18, 2005.

\end{thebibliography}
\appendix

\newpage

\section*{APPENDIX}

Let $\Z_{\geq 0}$ ($\Z_{\geq 1}$) be the set of non-negative (positive) integers. 

All results that are direct adaptations or restatements from existing results in the single RL literature from~\citep{CJ-ZY-ZW-MIJ:20} will have their detailed proofs if necessary to understand the nuances of their adaptation to our setting. Such proofs will be deferred to the last Section~\ref{sec:remaining-proofs} of the Appendix.

\section{Auxiliary results}

The following proposition is an immediate adaptation of an existing one in 
\citep{CJ-ZY-ZW-MIJ:20} for MDPs.

\begin{proposition}[Bounded parameters for Q-functions -- Proposition~2.3 and Lemma~B.1 in~\citep{CJ-ZY-ZW-MIJ:20}]
\label{prop:lin-Q}
Consider a linear stochastic game  $\M$. Given a policy profile $\pi$, we have that for any $i\in[n]$, there exist paremeters $w^{i,\pi}_h\in\R^d$, $h\in[H]$, such that $Q_h^{i,\pi}(x, a) = \langle\phi (x, a),w^{i,\pi}_h\rangle$ for any $(x, a) \in \S\times\A$ and $\norm{w^{i,\pi}_h} \leq 2H\sqrt{d}$.
\end{proposition}

The following lemma is a restatement of another one in~\citep{CJ-ZY-ZW-MIJ:20}, though with some different notation. 

\begin{lemma}[Concentration bound for self-normalized processes -- Lemma D.4 in~\citep{CJ-ZY-ZW-MIJ:20}]
\label{lem:self_norm_covering}
Let $\{\F_{\tau}\}_{\tau=0}^\infty$ be a filtration. Let $\{x_{\tau}\}_{\tau=1}^\infty$ be a stochastic process on $\S$ such that $x_{\tau}\in\F_{\tau}$, and let $\{\phi_{\tau}\}_{\tau=1}^\infty$ be an $\R^d$-valued stochastic process such that $\phi_{\tau} \in \F_{\tau-1}$ and $\norm{\phi_{\tau}}\leq 1$. Let $\G$ be a function class of real-valued functions such that $\sup_{x\in\S} |g(x)| \leq H$ for any $g\in\G$, and with $\epsilon$-covering number $\mathcal{N}_{\epsilon}$ with respect to the distance $\mathrm{dist}(g, g') = \sup_{x\in S} |g(x) - g'(x)|$. 
Let $\Lambda_{A} = \lambda I_d + \sum_{\tau=1}^A\phi_{\tau} \phi_{\tau}^\top$. Then for every $A\in \Z_{\geq 1}$, every $g \in \G$, and any $\delta\in(0,1]$, we have that with probability at least $1-\delta$,
\begin{equation*} 
\norm{\sum_{\tau = 1}^A\phi_{\tau} \{ g(x_{\tau}) - \E[g(x_{\tau})|\F_{\tau-1}] \} }^2_{\Lambda_{A}^{-1}}
\leq 4H^2 \left[ \frac{d}{2}\log\biggl( \frac{\lambda+A/d}{\lambda}\biggr )  + \log\frac{\mathcal{N}_{\epsilon}}{\delta}\right]  + \frac{8A^2\epsilon^2}{\lambda}.
\end{equation*}
\end{lemma}

The following lemma is a key result in the proof of Theorem~\ref{thm:main-nashQ}, as seen in Section~\ref{sec:NashQproof} from the main paper.

\begin{lemma}[Bounding the covering number]\label{lem:bound-import-app}
Let $i\in[n]$, and let $\bar{w}_i\in\R^d$ be such that $\norm{\bar{w}_i}\leq L$, $\bar{\Lambda}\in\R^{d\times d}$ be such that its minimum eigenvalue is greater or equal than $\lambda$, and, for all $(x,a)\in\S\times\A$, let $\phi(x,a)\in\R^d$ be such that $\norm{\phi(x, a)}\leq 1$, and let $\beta>0$.
Define the function class 
\begin{equation}
\label{eq:function_class}
\mathcal{V}_i = \Big\{V:\S\to\R\;\Big|\; V(\cdot)
=\max_{\nu\in\Delta(\A_i)}\E_{\substack{a_i\sim\nu\\a_{-i}\sim\pi_{-i}(\cdot)}}\Big[\min\Big\{ \bar{w}_i^\top\phi(\cdot, a)
+ \beta \sqrt{\phi(\cdot, a)^\top\bar{\Lambda}^{-1} \phi(\cdot, a)}, H \Big\}\Big]\Big\},
\end{equation}
where 
$\pi_{-i}(\cdot)\in\Delta(\A_{-i})$. Let $\mathcal{N}_{\epsilon}$ be the $\epsilon$-covering number of $\mathcal{V}_i$ with respect to the distance $\operatorname{dist}(V, V') = \sup_{x\in\S} |V(x) - V'(x)|$. Then, 
$$
\log \mathcal{N}_\epsilon \leq d  \log (1+ 4L/ \epsilon ) + d^2 \log \bigl [ 1 +  8 d^{1/2} \beta^2  / (\lambda\epsilon^2)  \bigr ].
$$
\end{lemma}
\begin{proof}
Let $V,V'\in\mathcal{V}_i$. Let $\bar{u}(x,a):=\sqrt{\phi(x, a)^\top\bar{\Lambda}^{-1} \phi(x,a)}$ and $\bar{u}'(x,a):=\sqrt{\phi(x, a)^\top(\bar{\Lambda}')^{-1} \phi(x, a)}$, and let $\bar{g}(x,a)=\min\left\{ \bar{w}_i^\top\phi(x, a) + \beta \bar{u}(x,a), H \right\}$ and $\bar{g}'(x,a)=\min\left\{ \bar{w}_i^\top\phi(x, a) + \beta \bar{u}'(x,a), H \right\}$. Then, 
\begin{equation}
\begin{aligned}
\dist(V,V')
&=\sup_{x\in\S}
\Big|\max_{\nu\in\Delta(\A_i)}\E_{\substack{a_i\sim\nu\\a_{-i}\sim\pi_{-i}(x)}}[\bar{g}(x,a)]-\max_{\nu\in\Delta(\A_i)}\E_{\substack{a_i\sim\nu\\a_{-i}\sim\pi_{-i}(x)}}[\bar{g}'(x,a)]\Big|\\
&\overset{(a)}{\leq}\sup_{\substack{x\in\S\\\nu\in\Delta(\A_i)}}
\Big|\E_{\substack{a_i\sim\nu\\a_{-i}\sim\pi_{-i}(x)}}[\bar{g}(x,a)]-\E_{\substack{a_i\sim\nu\\a_{-i}\sim\pi_{-i}(x)}}[\bar{g}'(x,a)]\Big|\\
&=\sup_{\substack{x\in\S\\\nu\in\Delta(\A_i)}}
\Big|\E_{\substack{a_i\sim\nu\\a_{-i}\sim\pi_{-i}(x)}}[\bar{g}(x,a)-\bar{g}'(x,a)]\Big|\\
&\leq\sup_{\substack{x\in\S\\\nu\in\Delta(\A_i)}}
\E_{\substack{a_i\sim\nu\\a_{-i}\sim\pi_{-i}(x)}}|\bar{g}(x,a)-\bar{g}'(x,a)|\\
&\leq\sup_{\substack{x\in\S\\a\in\A}}
\Big|\bar{g}(x,a)-\bar{g}'(x,a)\Big|\\
&\leq\sup_{\substack{x\in\S\\a\in\A}}|\min\{ \bar{w}_i^\top\phi(x, a) + \beta \bar{u}(x,a), H \} - \min\{ (\bar{w}_i')^\top\phi(x, a) + \beta \bar{u}'(x,a), H \}|\\
&\overset{(b)}{\leq}\sup_{\substack{x\in\S\\a\in\A}}|\bar{w}_i^\top\phi(x, a) + \beta \bar{u}(x,a) - ((\bar{w}_i')^\top\phi(x, a) + \beta \bar{u}'(x,a))|\\
&\leq \sup_{\substack{x\in\S\\a\in\A}}|(\bar{w}_i-\bar{w}_i')^\top\phi(x,a)|+\beta\sup_{\substack{x\in\S\\a\in\A}}|\bar{u}(x,a)-\bar{u}'(x,a)|
\end{aligned}
\end{equation}
Inequality (a) follows from the property $|\max_{\nu\in\Delta(\A_i)}f(\nu)-\max_{\nu\in\Delta(\A_i)}h(\nu)|\leq \max_{\nu\in\Delta(\A_i)}|f(\nu)-h(\nu)|$ for any $f,h:\Delta(\A_i)\to\R$ since, letting $\bar{\nu}=\argmax_{\nu\in\Delta(\A_i)}f(\nu)$ and $\tilde{\nu}=\argmax_{\nu\in\Delta(\A_i)}h(\nu)$, we observe that:  (i) if $\max_{\nu\in\Delta(\A_i)}f(\nu)>\max_{\nu\in\Delta(\A_i)}h(\nu)$, then $\max_{\nu\in\Delta(\A_i)}f(\nu)-\max_{\nu\in\Delta(\A_i)}h(\nu)\leq f(\bar{\nu})-h(\bar{\nu})\leq\max_{\nu\in\Delta(\A_i)}|f(\nu)-h(\nu)|$; and (ii) if $\max_{\nu\in\Delta(\A_i)}f(\nu)\leq\max_{\nu\in\Delta(\A_i)}h(\nu)$, then $\max_{\nu\in\Delta(\A_i)}h(\nu)-\max_{\nu\in\Delta(\A_i)}f(\nu)\leq h(\tilde{\nu})-f(\tilde{\nu})\leq\max_{\nu\in\Delta(\A_i)}|f(\nu)-h(\nu)|$. Inequality (b) follows from $\min\{\cdot,H\}$ being a non-expansive operator. 

We can continue bounding,
\begin{equation}
\label{eq:dist-covering-MG-0-app}
\begin{aligned}
\dist(V,V')
&\overset{(a)}{\leq}\sup_{\phi:\norm{\phi}\leq 1}\Big|
(\bar{w}_i-\bar{w}_i')^\top\phi\Big|\\
&\quad +\sup_{\phi:\norm{\phi}\leq 1}\beta\Big|
\sqrt{\phi^\top\bar{\Lambda}^{-1}\phi}-
\sqrt{\phi^\top(\bar{\Lambda}')^{-1}\phi}
\Big|\\
&\overset{(b)}{\leq}
\norm{\bar{w}_i-\bar{w}_i'} + \sup_{\phi:\norm{\phi}\leq 1}
\beta\sqrt{\left|\phi(x,a)^\top(\bar{\Lambda}^{-1}-(\bar{\Lambda}')^{-1})\phi(x,a)\right|}\\
&=\norm{\bar{w}_i-\bar{w}_i'} + 
\beta\sqrt{\norm{\bar{\Lambda}^{-1}-(\bar{\Lambda}')^{-1}}}\\
&\leq\norm{\bar{w}_i-\bar{w}_i'} + 
\beta\sqrt{\norm{\bar{\Lambda}^{-1}-(\bar{\Lambda}')^{-1}}_F}
\end{aligned}
\end{equation}
where (a) follows from the assumption $\sup_{x\in\S}\max_{a\in\A}\norm{\phi(x,a)}\leq 1$, and (b) follows from the inequality $|\sqrt{p}-\sqrt{q}|\leq\sqrt{|p-q|}$ for any $p,q\geq 0$. Now, we notice that~\eqref{eq:dist-covering-MG-0-app} is a bound of the same form of equation~(28) from~\citep[Lemma~D.6]{CJ-ZY-ZW-MIJ:20}, and so we can use the proof of this lemma to obtain that the $\epsilon$-covering number of $\mathcal{V}_i$, denoted by $\mathcal{N}_{\epsilon}$
can be upper bounded as $\log \mathcal{N}_{\epsilon}  \leq d  \log (1+ 4L/ \epsilon ) + d^2 \log \bigl [ 1 +  8 d^{1/2} \beta^2  / (\lambda\epsilon^2)  \bigr ]$. 
This finishes the proof.
\end{proof}

%
%


\section{Proving Theorem~\ref{thm:main-nashQ}}
\label{subsec:proof_main_parallel}
For simplicity, we will use the following notation: at episode $k$, we denote $\pi^{i,k}=\{\pi^{i,k}_h\}_{h\in[H]}$ as the policy induced by $\{Q_h^{i,k}\}_{h=1}^H$ as performed by agent $i\in[n]$ (line 14 of Algorithm~1) across time steps $h\in[H]$, thus for a fixed step $h\in[H]$ we let $V_h^{i,k}(x_h^{k})= \E_{a\sim \pi^{k}_h(x_h^{k})}[Q_h^{i,k}(x_h^{k},a)]$ with $\pi^{k}_h(x_h^{k})$ being a Nash equilibrium from the stage game $(Q^{i,k}_h(x_h^k,\cdot))_{i\in[n]}$. With some abuse of notation, we similarly define $V_h^{i,k}(x)= \E_{a\sim \pi^{k}_h(x)}[Q_h^{i,k}(x,a)]$ with $\pi^{k}_h(x)$ being a Nash equilibrium from the game $(Q^{i,k}_h(x,\cdot))_{i\in[n]}$. Let $\phi^\tau_h:=\phi(x^\tau_h,a^\tau_h)$. 

\subsection{Preliminary technical results}

We now bound the parameters $\{w^{i,k}_h\}_{(i,h,k)\in[n]\times[H]\times[K]}$ from the \algbrev ~algorithm.

\begin{lemma}[Parameter bound -- Lemma~B.2 in~\citep{CJ-ZY-ZW-MIJ:20}]
\label{lem:wn_estimate}
For any $(i,k, h) \in[n]\times[K]\times[H]$, the parameter $w^{i,k}_h$ in the \algbrev ~algorithm satisfies
$\norm{w^{i,k}_h}\leq(1+H) \sqrt{\frac{d(k-1)}{\lambda}}$.
\end{lemma}

Now we use Lemma~\ref{lem:bound-import-app} and Lemma~\ref{lem:self_norm_covering} to prove a useful concentration bound for \algbrev.

\begin{lemma}[Concentration bound on value functions for \algbrev ~-- Lemma~B.3 in~\citep{CJ-ZY-ZW-MIJ:20}] \label{lem:stochastic_term}
Consider the setting of Theorem~\ref{thm:main-nashQ}. There exists an absolute constant $C$ independent of $c_{\beta}$ such that for any fixed $\delta\in(0, 1)$, the following event $\mathcal{E}_i$ holds with probability at least $1-\delta$ for a fixed $i\in[n]$: for every $(k, h)\in [K]\times [H]$,
\begin{equation*}
\norm{\sum_{\tau = 1}^{k-1} \phi^{\tau}_h [V^{i,k}_{h+1}(x^{\tau}_{h+1}) - \Pe_h V^{i,k}_{h+1}(x_h^{\tau}, a_h^{\tau})]}_{(\Lambda^k_h)^{-1}}
\leq CdH\sqrt{\log [(c_\beta+1)dKH/\delta]}. 
\end{equation*}
\end{lemma}

The following lemma crucially depends on the principle of optimism.  

\begin{lemma}[Difference with an arbitrary Q-function -- Lemma~B.4 in~\citep{CJ-ZY-ZW-MIJ:20}]
\label{lem:basic_relation} Consider the setting of Theorem~\ref{thm:main-nashQ}. There exists an absolute constant $c_\beta$ such that for $\beta = c_\beta dH\sqrt{\iota}$ with $\iota = \log (dKH/\delta)$ and
any fixed joint policy $\bar{\pi}$, such that for any $i\in[n]$: given the event $\mathcal{E}_i$ defined in Lemma \ref{lem:stochastic_term}, we have for all $(x, a, h, k) \in \S\times\A\times[H]\times[K]$ that
\begin{equation*}
\langle\phi(x, a), w^{i,k}_h\rangle - Q_h^{i,\bar{\pi}}(x, a)  =  \Pe_h (V^{i,k}_{h+1} - V^{i,\bar{\pi}}_{h+1})(x, a) + \Delta^{i,k}_h(x, a),
\end{equation*}
for some $\Delta^{i,k}_h(x, a)$ such that $|\Delta^{i,k}_h(x, a)| \leq \beta \sqrt{\phi(x, a)^\top (\Lambda^k_h)^{-1}  \phi(x, a)}$.
\end{lemma}

The following key lemma makes use of optimism by using Lemma~\ref{lem:basic_relation} and of the fact that we choose a Nash equilibrium at each stage game.

\begin{lemma}[Optimism bounds]
\label{lem:optimism_bound_app} Consider the setting of Theorem~\ref{thm:main-nashQ}. Given the event $\mathcal{E}_i$ defined in Lemma \ref{lem:stochastic_term}, we have that for all $(x, a, h, k) \in \S\times\A\times[H]\times[K]$, $$Q^{i,\bre(\pi_{-i}^k),\pi_{-i}^k}_h(x,a)\leq Q^{i,k}_h(x,a)\quad \text{ and }\quad V^{i,\bre(\pi_{-i}^k),\pi_{-i}^k}_h(x)\leq V^{i,k}_h(x).$$
\end{lemma}
\begin{proof}
We prove the claims by induction in $h=H+1,\dots,1$. The base case $H+1$ is trivial, since $Q^{i,\bre(\pi_{-i}^k),\pi_{-i}^k}_{H+1}(x,a)= Q^{i,k}_{H+1}(x,a)=0$. Now, at step $h+1$ we have the induction hypothesis $Q^{i,\bre(\pi_{-i}^k),\pi_{-i}^k}_{h+1}(x,a)\leq Q^{i,k}_{h+1}(x,a)$. Then we have that 
\begin{equation}
\label{eq:boundV-at-h1}
\begin{aligned}
V^{i,\bre(\pi_{-i}^k),\pi_{-i}^k}_{h+1}(x)&\overset{(a)}{=}\max_{\nu\in\Delta(\A_i)}\E_{\substack{a_i\sim\nu\\a_{-i}\sim\pi^k_{-i,h+1}(x)}}[Q^{i,\bre(\pi_{-i}^k),\pi_{-i}^k}_{h+1}(x,a)]\\
&\overset{(b)}{\leq}\max_{\nu\in\Delta(\A_i)}\E_{\substack{a_i\sim\nu\\a_{-i}\sim\pi^k_{-i,h+1}(x)}}[Q^{i,k}_{h+1}(x,a)]\\
&\overset{(c)}{=}\E_{a\sim\pi^k_{h+1}(x)}[Q^{i,k}_{h+1}(x,a)]\\
&=V^{i,k}_{h+1}(x,a),
\end{aligned}
\end{equation}
where (a) follows by definition of best response, (b) from the induction hypothesis, and also (a)-(c) altogether from the fact that \algbrev ~chooses a Nash equilibrium at every stage game.

Now, we have
\begin{equation}
\begin{aligned}
Q^{i,\bre(\pi_{-i}^k),\pi_{-i}^k}_{h}(x,a)&\overset{(a)}{\leq}
\langle\phi(x, a), w^{i,k}_h\rangle +  \Pe_h (V^{i,k}_{h+1} - V^{i,\bre(\pi_{-i}^k),\pi_{-i}^k}_{h+1})(x, a) + \beta \sqrt{\phi(x,a)^\top (\Lambda^k_h)^{-1}  \phi(x,a)}\\
&\overset{(b)}{\leq} 
\langle\phi(x, a), w^{i,k}_h\rangle + \beta \sqrt{\phi(x,a)^\top (\Lambda^k_h)^{-1}  \phi(x,a)}\\
\overset{(c)}{\implies}Q^{i,\bre(\pi_{-i}^k),\pi_{-i}^k}_{h}(x,a)&\leq \min\{\langle\phi(x, a), w^{i,k}_h\rangle + \beta \sqrt{\phi(x,a)^\top (\Lambda^k_h)^{-1}  \phi(x,a)},H\}\\
&=Q^{i,k}_h(x,a),
\end{aligned}
\end{equation}
where (a) follows from Lemma~\ref{lem:basic_relation}, (b) from~\eqref{eq:boundV-at-h1}, and (c) from $Q^{i,\bre(\pi_{-i}^k),\pi_{-i}^k}_{h}\leq H$. From here we can repeat the steps in~\eqref{eq:boundV-at-h1} to obtain $V^{i,\bre(\pi_{-i}^k),\pi_{-i}^k}_h(x)\leq V^{i,k}_h(x)$. This finishes the proof. 
\end{proof}

\subsection{Proof of Theorem~\ref{thm:main-nashQ}}
Let us first condition on the event $\bigcap_{i=1}^n\mathcal{E}_i$ where $\mathcal{E}_i$ is defined in Lemma~\ref{lem:stochastic_term}. Since $\Pe[\text{not }  \mathcal{E}_i]\leq \delta$, applying union bound let us conclude that $\Pe[\bigcap_{i\in[n]}\mathcal{E}_i]\geq 1-n\delta
$.

For any $k\in[K]$, given the policy $\pi^k=\{\pi^k_i\}_{i\in[n]}$ defined by \algbrev, we define the functions $\hat{Q}_h^k$ 
and $\hat{V}_h^k$ 
recursively as: $\hat{V}_{H+1}^k(x)=\hat{Q}_{H+1}^k(x)=0$ and 
\begin{align*}
\hat{Q}_h^k(x,a) &= 
\Pe_h\hat{V}^k_{h+1}(x,a)+2\beta\sqrt{(\phi^k_h)^\top(\Lambda^k_h)^{-1}\phi^k_h},\\
\hat{V}_h^k(x)&=\E_{a\sim\pi^k_h}[\hat{Q}_h^k(x,a)]
\end{align*}
for any $h= H,\dots,1$ and $(x,a)\in\S\times\A$. Notice that since $2\beta\sqrt{(\phi^k_h)^\top(\Lambda^k_h)^{-1}\phi^k_h}\leq 2\beta\sqrt{(\phi^k_h)^\top\phi^k_h}=2\beta\norm{\phi^k_h}\leq 2\beta$, we have that $\hat{Q}_h^k$ 
and $\hat{V}_h^k$ are nonnegative with maximum value $2\beta H$.

Let $k\in[K]$. We claim that for any $(h,x,a)\in[H]\times\S\times\A$, 
\begin{equation}
\label{eq:claim-upper-bound-app}
\begin{aligned}
        \max_{i\in[n]}(Q^{i,k}_h(x,a)-Q^{i,\pi^k}_h(x,a))&\leq \hat{Q}_h^k(x,a)\text{, and}\\
        \max_{i\in[n]}(V^{i,k}_h(x)-V^{i,\pi^k}_h(x))&\leq \hat{V}_h^k(x).
\end{aligned}
\end{equation}
We prove the claim by induction in $h= H+1,\dots,1$. The base case $H+1$ is trivial, since $Q^{i,k}_{H+1}(x,a)=Q^{i,\pi^k}_{H+1}(x,a)=\hat{Q}^k_{H+1}(x,a)=0$ for every $i\in[n]$. Now, at step $h+1$ we have the induction hypothesis $\max_{i\in[h]}(Q^{i,k}_{h+1}(x,a)-Q^{i,\pi^k}_{h+1}(x,a))\leq \hat{Q}_{h+1}^k(x,a)$. Taking expectations over $a\sim\pi^k_{h+1}(x)$ let us immediately obtain 
\begin{equation}
\label{eq:max_Vi}
\max_{i\in[h]}(V^{i,k}_{h+1}(x)-V^{i,\pi^k}_{h+1}(x))\leq \hat{V}_{h+1}^k(x).    
\end{equation} 

Now, for any $i\in[n]$,
\begin{equation}
\begin{aligned}
Q^{i,k}_h(x,a)-Q^{i,\pi^k}_h(x,a)&= \min\{(w^{i,k}_h)^\top\phi(x,a)+\beta\sqrt{\phi(x,a)^\top(\Lambda_h^k)^{-1}\phi(x,a)},H\}-Q^{i,\pi^k}_h(x,a)\\
&\overset{(a)}{\leq}
\Pe_h (V^{i,k}_{h+1} - V^{i,\pi^k}_{h+1})(x, a) + 2\beta \sqrt{\phi(x,a)^\top (\Lambda^k_h)^{-1}  \phi(x,a)}\\
&\overset{(b)}{\leq} 
\Pe_h\hat{V}^{k}_{h+1}(x,a) + 2\beta \sqrt{\phi(x,a)^\top (\Lambda^k_h)^{-1}  \phi(x,a)}\\
&=\hat{Q}^{k}_h(x,a),
\end{aligned}
\end{equation}
where (a) follows from Lemma~\ref{lem:basic_relation} and (b) from \eqref{eq:max_Vi}. Taking expectations let us obtain  $V^{i,k}_h(x)-V^{i,\pi^k}_h(x)\leq \hat{V}^{k}_h(x)$. This finishes the proof for the claim in~\eqref{eq:claim-upper-bound-app}.

We now introduce the following notation: $\delta^{k}_h := \E_{a\sim\pi^k_h}[\hat{Q}^k_h(x^k_h,a)]-\hat{Q}^k_h(x^k_h,a^k_h)$, and $\xi^k_{h+1} := 
\Pe_h\hat{V}^k_{h+1}(x^k_h,a^k_h) - \hat{V}^k_{h+1}(x^k_{h+1})$ with $\xi^k_1:=0$. Then, for any $(h,k) \in [H] \times [K]$, 
\begin{align*}
    \hat{V}^k_h(x^k_h)&=\E_{a\sim\pi^k_h(x^k_h)}[\hat{Q}^k_h(x^k_h,a)]\\
    &=\delta^k_h+\hat{Q}^k_h(x^k_h,a^k_h)\\
    &=\delta^k_h+\Pe_h\hat{V}^k_{h+1}(x^k_h,a^k_h)+2\beta\sqrt{(\phi^k_h)^\top(\Lambda^k_h)^{-1}\phi^k_h}\\
    &=\delta^k_h+\xi^k_{h+1}+2\beta\sqrt{(\phi^k_h)^\top(\Lambda^k_h)^{-1}\phi^k_h}+\hat{V}^k_{h+1}(x^k_{h+1}).
\end{align*}

Now, let us focus on the regret performance metric.
\begin{equation}
\label{eq:regret_prev_p_app}
\begin{aligned}
\textnormal{Regret}(K) &=  \sum_{k=1}^K\max_{i\in[n]}(V_1^{i,\bre(\pi^k_{-i}),\pi^k_{-i}}(s_o) - V_1^{i,\pi^k}(s_o))\\
&\overset{(a)}{\leq}\sum_{k=1}^{K} \max_{i\in[n]}(V^{i,k}_1(s_o) - V^{i,\pi^k}_1 (s_o))\\
&\overset{(b)}{\leq}\sum_{k=1}^{K} \hat{V}^{k}_1(s_o)
\\
&=\underbrace{\sum_{k=1}^K\sum_{h=1}^H\xi^k_h}_{\textrm{(I)}}+\underbrace{\sum_{k=1}^{K}\sum_{h=1}^H \delta^{k}_{h}}_{\textrm{(II)}} + \underbrace{2\beta \sum_{k=1}^{K}\sum_{h=1}^H \sqrt{(\phi^{k}_h)^\top (\Lambda^k_h)^{-1}\phi^{k}_h}}_{\textrm{(III)}},
\end{aligned}
\end{equation}
where (a) follows from Lemma~\ref{lem:optimism_bound_app} and the fact that we are conditioned on the event $\bigcap^n_{i=1}\mathcal{E}_i$; (b) follows from~\eqref{eq:claim-upper-bound-app}.

We first analyze the term (I) from~\eqref{eq:regret_prev_p_app}. 
Let us define the filtration $\{\F_{(k,h)}\}_{(k,h)\in\L^\star}$ where
$\L^\star$ is a sequence such that $\L^\star\subset\mathbb{Z}_{\geq 1}\times[H]$ and its elements are arranged as follows. Firstly, we let the second coordinate take values from $1$ to $H$ and repeat this periodically \emph{ad infinitum}, so that each period has $H$ elements of $\L^\star$. Finally, we let the first coordinate take the value corresponding to the current number of periods so far in the second coordinate (and so its value is unbounded). 
Consider any element $(k,h)\in\L^\star$. We denote by $(k,h)^{-1}$ its previous element in $\L^\star$. We let $\F_{(k,h)}$ contain the information of the tuple $(x^{\bar{k}}_{\bar{h}},a^{\bar{k}}_{\bar{h}})$ whose indexes $(\bar{k},\bar{h})$ belong to the set $\L^\star$ up to the element $(k,h)\in\L^\star$.

We then can conclude that $\{\xi^k_h\}_{(k,h)\in\L^\star}$ is a 
martingale difference sequence due to the following two properties:
\begin{enumerate}
    \item $\xi^k_h\in\F_{(k,h)^{-1}}$.
    For $h=1$, $\E[\xi^k_{h}|\F_{(k,h)^{-1}}]=0$ is trivial, so we focus on $h=2,\dots,H$. Then, since $x^{k}_h\sim\P_{h-1}(\cdot|x^{k}_{h-1},a^{k}_{h-1})$ (line 16 of \algbrev), we have $\E[\hat{V}^k_h(x^k_h)|\F_{(k,h)^{-1}}]=\E_{x'\sim\P_{h-1}(\cdot|x_{h-1}^{k},a_{h-1}^{k})}[\hat{V}^{k}_h(x')]=\Pe_{h-1}\hat{V}^k_h(x^k_{h-1},a^k_{h-1})$, which immediately implies $\E[\xi^k_{h}|\F_{(k,h)^{-1}}]=0$. 
    \item $|\xi^k_h|\leq 
    |\Pe_{h-1}\hat{V}^k_h(x^{k}_{h-1}, a^{k}_{h-1})| +|\hat{V}^k_h(x^k_h)|\leq 4\beta H <\infty$
    since $\hat{V}^k_h(x)\in[0,2\beta H]$ for any $x\in\S$.
\end{enumerate}
Therefore, we can use the Azuma-Hoeffding inequality to conclude that, for any $\epsilon > 0$,
\begin{equation*}
\Pr \left(\sum_{k=1}^{K}\sum_{h=1}^H  \xi^{k}_{h}> \epsilon \right) \leq \exp \bigg (\frac{-2 \epsilon^2 } {(KH)(16\beta^2H^2) } \bigg ).
\end{equation*}
We choose $\epsilon=\sqrt{8KH^3\beta^2\log\left(\frac{1}{\delta}\right)}$. Then, with probability at least $1 -\delta$,   
\begin{equation}
\label{eq:final2-app}
  \textrm{(I)}=\sum_{k=1}^{K}\sum_{h=1}^H  \xi^k_{h}\leq \sqrt{8KH^3\beta^2\log\left(\frac{1}{\delta}\right)} \leq 8\beta H\sqrt{KH\iota}, 
\end{equation}
recalling that $\iota = \log\left(\frac{dKH}{\delta}\right)$. We call $\bar{\mathcal{E}}$ the event such that~\eqref{eq:final2-app} holds.

The term (II) can be analyzed in a very similar way as in (I) to show that $\{\delta^k_h\}_{(k,h)\in\cL^*}$ is a martingale difference sequence, and thus obtain that with probability at least $1-\delta$,
\begin{equation}
\label{eq:final3-app}
  \textrm{(II)}=\sum_{k=1}^{K}\sum_{h=1}^H  \delta^k_{h}\leq 8\beta H\sqrt{KH\iota}. 
\end{equation}
We call $\tilde{\mathcal{E}}$ the event such that~\eqref{eq:final3-app} holds.

We now analyze the term (III) from~\eqref{eq:regret_prev_p_app}. Then for a fixed $h\in[H]$,
$$
\sum_{k=1}^{K}(\phi^{k}_h)^\top (\Lambda^k_h)^{-1}\phi^{k}_h
%
\leq 2\log\left[\frac{\det(\Lambda_h^{K+1})}{\det(\lambda I_d)}\right]
$$
where the inequality follows from the so-called elliptical potential lemma~\citep[Lemma~11]{YAY-DP-CS:11}, whose conditions are satisfied from our bounded sequence $\{\phi_h^{k}\}_{k=1}^K$ and the fact that the minimum eigenvalue of $\Lambda_h^k$ is lower bounded by $\lambda=1$ for every $(h,k)\in[H]\times[K]$. Now, we have that $\Lambda_h^{K+1}$ is a positive definite matrix whose maximum eigenvalue can be bounded as 
$\norm{\Lambda_h^{K+1}}\leq \norm{\sum_{k=1}^K\phi^{k}_h(\phi^{k}_h)^\top}+\lambda\leq K+\lambda$, and so $\det(\Lambda^{K+1}_h)\leq \det((K+\lambda)I_d)=(K+\lambda)^d$. We also have that $\det(\lambda I_d)=\lambda^d$. Then, we obtain that
\begin{equation}
\label{eq:aux_last_1}
\sum_{k=1}^{K}(\phi^{k}_h)^\top (\Lambda^k_h)^{-1}\phi^{k}_h
\leq
2\log\left[\frac{K+\lambda}{\lambda}\right]^d=
2d\log(K+1)\leq 2d\iota,
\end{equation}
where the last inequality holds since $\log(K+1)\leq \log\left(\frac{dKH}{\delta}\right)=\iota$ for $d\geq 2$. 

Now, going back to term (III), 
\begin{equation}
\label{eq:aux_last_2-app}
\textrm{(III)}=2\beta \sum_{h=1}^H\sum_{k=1}^{K} \sqrt{(\phi^{k}_h)^\top (\Lambda^k_h)^{-1}\phi^{k}_h}
  \overset{\textrm{(a)}}{\leq}2\beta \sum_{h=1}^H\sqrt{K}\sqrt{ \sum_{k=1}^K (\phi^{k}_h)^\top (\Lambda^k_h)^{-1}\phi^{k}_h}
\overset{(b)}{\leq} 2\beta H \sqrt{2dK\iota}, 
\end{equation}
where (a) follows from the Cauchy-Schwartz inequality, and (b) from~\eqref{eq:aux_last_1}.

Now, using the results in~\eqref{eq:final2-app}, \eqref{eq:final3-app}, and~\eqref{eq:aux_last_2-app} back in~\eqref{eq:regret_prev_p_app}, we conclude that,
\begin{multline}
\label{eq:regret_almost_l}
\textnormal{Regret}(K) \leq  
8\beta H\sqrt{KH\iota} + 8\beta H\sqrt{KH\iota}
+ 
2\beta H \sqrt{dK\iota}\\
=16c_\beta\sqrt{d^2 K H^5\iota^2}
+ 
2c_\beta\sqrt{ d^3KH^4\iota^2}
\overset{\textrm{(a)}}{\leq} 18c_\beta\sqrt{d^3KH^5\iota^2},
\end{multline}
where (a) follows from $\sqrt{\iota}\leq \iota$ which follows from equation~\eqref{eq:iota_bound}.

Finally, 
applying union bound let us conclude that 
$\Pe[\bigcap_{i\in[n]}\mathcal{E}_i\cap\bar{\mathcal{E}}\cap\tilde{\mathcal{E}}]\geq 1-(n+2)\delta
$,
i.e., our final result holds with probability at least $1-(n+2)\delta
$. This finishes the proof of Theorem~\ref{thm:main-nashQ}.\qed
%

\section{Remaining proofs}
\label{sec:remaining-proofs}

\begin{proof}[Proof of Lemma~\ref{lem:self_norm_covering}]
First, from our assumptions, for any $g \in \G$, there exists a $\tilde{g}$ in the $\epsilon$-covering such that
$g = \tilde{g} + \Delta_g$ with $\sup_{x\in\S} |\Delta_g(x)| \leq \epsilon$. Then,
\begin{equation}
\label{eq:conc-aux}
\begin{aligned}
&\norm{\sum_{\tau = 1}^A \phi_{\tau} \{ g(x_\tau) - \E[g(x_\tau)|\F_{\tau-1}] \}  }^2_{\Lambda_{A}^{-1}}\\
&\quad\leq  2\underbrace{\norm{\sum_{\tau = 1}^A \phi_{\tau} \{ \tilde{g}(x_{\tau}) - \E[\tilde{g}(x_{\tau})|\F_{\tau-1}]\} }^2_{\Lambda_{A}^{-1}}}_{\textrm{(I)}}
+2\underbrace{\norm{\sum_{\tau = 1}^A \phi_{\tau} \{ \Delta_g(x_{\tau}) - \E[\Delta_g(x_{\tau})|\F_{\tau-1}] \} }^2_{\Lambda_{A}^{-1}}}_{\textrm{(II)}},
\end{aligned}
\end{equation}
where we used $\norm{a+b}\leq \norm{a}+\norm{b}\implies \norm{a+b}^2\leq \norm{a}^2+\norm{b}^2+2\norm{a}\norm{b}\leq 2\norm{a}^2+2\norm{b}^2$ for any $a,b\in\R^d$, and which actually holds for any weighted Euclidean norm. 

We start by analyzing the term (I) in equation~\eqref{eq:conc-aux}. Let $\varepsilon_{\tau}:=\tilde{g}(x_{\tau}) - \E[\tilde{g}(x_{\tau})|\F_{\tau-1}]$. Now, we observe that 1) $\E[\varepsilon_{\tau}|\F_{\tau-1}]=0$ and 2) $\varepsilon_{\tau}\in[-H,H]$ since $\tilde{g}(x_{\tau})\in[0,H]$. From these two facts we obtain that $\varepsilon_{\tau}|\F_{\tau-1}$ is $H$-sub-Gaussian. Therefore we can apply the concentration bound of self-normalized processes from Theorem~1 of~\citep{YAY-DP-CS:11} along with a union bound over the $\epsilon$-covering of $\G$ to conclude that, with probability at least $1-\delta$,
\begin{multline}
\label{eq:upp-b-lem-aux1}
\textrm{(I)}=\norm{\sum_{\tau=1}^A\phi_{\tau}\varepsilon_{\tau}}_{\Lambda_{A}^{-1}}^2\leq \log\left(\frac{\det(\Lambda_{A})^{1/2}\det(\lambda I_d)^{-1/2}}{\delta/\N_\epsilon}\right)
\overset{\textrm{(a)}}{\leq} 2H^2\left(
\frac{d}{2}\log\left(\frac{\lambda+AB/d}{\lambda}\right)+\log\left(\frac{\N_\epsilon}{\delta}\right)
\right),
\end{multline}
where (a) follows from $\det(\lambda I_d)=\lambda^d$ and from the determinant-trace inequality from Lemma~10 in~\citep{YAY-DP-CS:11} which let us obtain $\det(\Lambda_{A})\leq(\lambda+AB/d)^d$.

Now we analyze the term (II) in equation~\eqref{eq:conc-aux}. Let $\bar{\varepsilon}_{\tau}:=\Delta_g(x_{\tau}) - \E[\Delta_g(x_{\tau})|\F_{\tau-1}]$. Then, 
\begin{equation*}
\norm{\sum_{\tau=1}^A\phi_{\tau}\bar{\varepsilon}_{\tau}}\leq \sum_{\tau=1}^A\norm{\phi_{\tau}\bar{\varepsilon}_{\tau}}\overset{\textrm{(a)}}{\leq} \sum_{\tau=1}^A|\bar{\varepsilon}_{\tau}|
\leq \sum_{\tau=1}^A|\Delta_g(x_{\tau})|+|\E[\Delta_g(x_{\tau})|\F_{\tau-1}]|\leq \sum_{\tau=1}^A 2\epsilon=2A\epsilon,
\end{equation*}
where (a) follows from $\norm{\phi_{\tau}}\leq 1$. Thus, using this result, we obtain
$$
\textrm{(II)}\leq\frac{1}{\lambda}\norm{\sum_{\tau=1}^A\phi_{\tau}\bar{\varepsilon}_{\tau}}^2\leq \frac{1}{\lambda}4A^2\epsilon^2.$$

We finish the proof by multiplying by two the terms (I) and (II), and then adding them up to use them as an upper bound to~\eqref{eq:conc-aux} .
\end{proof}

\begin{proof}[Proof of Lemma~\ref{lem:wn_estimate}]
For any vector $v \in \R^d$,
\begin{align*}
|v^\top w^{i,k}_h| & = |v^\top (\Lambda^k_h)^{-1} \sum_{\tau=1}^{k-1} \phi^{\tau}_h [r^i_h + \max_{\substack{a\sim\pi^*\\\pi^*\text{ as in line 7 of Algorithm~1}}} Q_{h+1}^{i,k}(x^{\tau}_{h+1}, a)]|\\
& \overset{\textrm{(a)}}{\leq}(1+H)\sum_{\tau = 1}^{k-1}  |v^\top (\Lambda^k_h)^{-1} \phi^{\tau}_h|\\
&\overset{(b)}{\leq} (1+H)\sqrt{ \bigg[ \sum_{\tau = 1}^{k-1}  v^\top (\Lambda^k_h)^{-1}v\bigg]  \biggl [ \sum_{\tau = 1}^{k-1}  (\phi^\tau_h)^\top (\Lambda^k_h)^{-1}\phi^\tau_h\bigg] }\\
& \overset{\textrm{(c)}}{\leq} (1+H)\sqrt{d}\sqrt{\sum_{\tau = 1}^{k-1}  v^\top (\Lambda^k_h)^{-1}v}\\
&\overset{\textrm{(d)}}{\leq}(1+H)\sqrt{\frac{d(k-1)}{\lambda}}\norm{v}, %
\end{align*}
where (a) follows from the bounded rewards and $Q^{i,k}_{h+1}(\cdot,\cdot)\leq H$; (b) from applying Cauchy-Schwarz twice as in the following series of inequalities: given $q = (q_1,\dots,q_m)$ and $q = (p_1,\dots,p_m)$ where $q_i$ and $p_i$ are vectors of same arbitrary dimension we have $\sum^m_{i=1}|q_i^\top p_i|\leq \sum^m_{i=1}\norm{q_i}\norm{p_i}\leq \sqrt{\sum^m_{i=1}\norm{q_i}}\sqrt{\sum^m_{i=1}\norm{p_i}}$ ; (c) follows from~\cite[Lemma~D.1]{CJ-ZY-ZW-MIJ:20}; and (d) from $(\Lambda_h^k)^{-1}\preceq \lambda^{-1}I_d$. The proof concludes by considering that $\norm{w^{i,k}_h} = \max_{v:\norm{v} = 1} |v^\top w^{i,k}_h|$.
\end{proof}

\begin{proof}[Proof of Lemma~\ref{lem:stochastic_term}]
We obtain that, with probability at least $1-\delta$, $\delta\in(0,1)$,
\begin{equation}
    \label{eq:bound_aux_long}
\begin{aligned}
&\norm{\sum_{\tau = 1}^{k-1} \phi^{\tau}_h [V^{i,k}_{h+1}(x^{\tau}_{h+1}) - \Pe_h V^{i,k}_{h+1}(x_h^{\tau}, a_h^{\tau})]}_{(\Lambda^k_h)^{-1}}^2\\
&\overset{\textrm{(a)}}{\leq} 4H^2 \left[ \frac{d}{2}\log\biggl( \frac{\lambda+(k-1)/d}{\lambda}\biggr )  + \log\N_\epsilon + \log\frac{1}{\delta}
\right]  + \frac{8(k-1)^2\epsilon^2}{\lambda}\\
&\overset{(b)}{\leq} 4H^2 \left[ \frac{d}{2}\log\biggl( \frac{\lambda+(k-1)/d}{\lambda}\biggr )  
+ d\log \left(1+ \frac{4(1+H) \sqrt{d(k-1)}}{\epsilon\sqrt{\lambda}}\right) \right. \\
&\quad\left. + d^2 \log\left( 1 + \frac{8 d^{1/2}\beta^2}{\lambda\epsilon^2}\right) + \log\frac{1}{\delta}
\right]  + \frac{8(k-1)^2\epsilon^2}{\lambda}
\end{aligned}
\end{equation}
where (a) is a direct application of Lemma~\ref{lem:self_norm_covering}; and (b) follows from the realization that, from lines 9 and 10 in Algorithm~1, $V^{i,k}_{h+1}(\cdot)\in\mathcal{V}$ with $\mathcal{V}$ as in Lemma~\ref{lem:bound-import-app} and so we can use the bound on the covering number derived in such lemma with $L=(1+H) \sqrt{\frac{d(k-1)}{\lambda}}$ by using the bound from Lemma~\ref{lem:wn_estimate}.

Recalling that $\lambda = 1$ and $\beta=c_\beta dH\iota$ with $\iota=\log(dKH/\delta)$ in the setting of Theorem~\ref{thm:main-nashQ}, we claim that, after setting $\epsilon = \frac{dH}{K}$ in our previous equation, there exists an absolute constant $C > 0$ independent of $c_\beta$ such that 
\begin{equation}
\label{eq:final-bound-1}
\norm{\sum_{\tau = 1}^{k-1} \phi^{\tau}_h [V^{i,k}_{h+1}(x^{\tau}_{h+1}) - \Pe_h V^{i,k}_{h+1}(x_h^{\tau}, a_h^{\tau})]}_{(\Lambda^k_h)^{-1}}^2 \leq C d^2  H^2 \log ((c_\beta+1)dKH/
\delta).
\end{equation}
Proving~\eqref{eq:final-bound-1} would conclude the proof.

We first introduce a couple of useful results:
\begin{align}
\label{eq:iota_bound}
&\iota^2=\log\left(\frac{dKH}{\delta}\right)\geq \log(dKH)\geq \log(4)>1,\\ 
\label{eq:upp_low_bound}
& \log \left(\frac{(c_\beta+1)dKH}{\delta}\right)= 
\log(c_\beta+1)+\iota\geq \iota>1.
\end{align}
Replacing $\lambda =1$ and $\epsilon=\frac{dH}{K}$ in the right-hand side of~\eqref{eq:bound_aux_long} and doing some algebraic calculations, let us conclude that
\begin{equation}
\label{eq:aux-first-res}
\begin{aligned}
\eqref{eq:bound_aux_long}
&\leq 4d^2H^2\left[\log\left( 1+\frac{K}{d}\right)  
+ \log \left(1+ \frac{8K^{3/2}}{d^{1/2}}\right) + \log\left(\frac{1}{\delta}\left( 1 + \frac{8 \beta^2K^2}{d^{3/2}H^2}\right)\right) \right] +8d^2H^2.
\end{aligned}
\end{equation}
Replacing $\beta=c_\beta dH\sqrt{\iota}$ in the previous expression and doing some algebraic work let us obtain
\begin{equation}
\label{eq:to_upper_bound}
\begin{aligned}
\eqref{eq:aux-first-res}&\leq \underbrace{8d^2H^2\log \left(1+ \frac{8K^{3/2}}{d^{1/2}}\right)}_{\textrm{(I)}} + 
\underbrace{4d^2H^2\log\left(\frac{1}{\delta}\left( 1 + 8 c_\beta^2d^{1/2}\iota K^2\right)\right)}_{\textrm{(II)}} 
+8d^2H^2\log\left(\frac{(c_\beta+1)dKH}{\delta}\right)
\end{aligned}
\end{equation}
where the inequality has made use of~\eqref{eq:upp_low_bound}. We now upper bound the terms highlighted in~\eqref{eq:to_upper_bound}.
Then, 
\begin{align*}
\textrm{(I)}&\leq 8d^2H^2\log \left(1+ 8K^{3/2}\right)\\
&\overset{\textrm{(a)}}{\leq} 8d^2H^2\log \left(\frac{(1+c_\beta)^2(dKH)^2}{\delta^2}\right)+8d^2H^2\log(9)\log \left(\frac{(c_\beta+1)dKH}{\delta}\right)\\
&=(16+8\log(9))d^2H^2\log \left(\frac{(c_\beta+1)dKH}{\delta}\right),
\end{align*}
where (a) follows from~\eqref{eq:upp_low_bound} and $c_\beta>0$. Similarly,
\begin{align*}
\textrm{(II)}
&\overset{\textrm{(a)}}{\leq} 
4d^2H^2\log\left(\frac{8(c_\beta+1)^2\iota(dKH)^2}{\delta}\right)\\
&\overset{(b)}{\leq} 
4d^2H^2\log\left(\frac{(c_\beta+1)^2\iota(dKH)^2}{\delta^2}\right)+4d^2H^2\log(8)\\
&=
4d^2H^2\log\left(\frac{(c_\beta+1)^2(dKH)^2}{\delta^2}\right)+4d^2H^2\log(\iota)+4d^2H^2\log(8)\\
&\overset{\textrm{(c)}}{\leq} 8d^2H^2\log\left(\frac{(c_\beta+1)dKH}{\delta}\right)+4d^2H^2\iota+4d^2H^2\log(8)\\
&\overset{\textrm{(d)}}{\leq} (12+4\log(8))d^2H^2\log\left(\frac{(c_\beta+1)dKH}{\delta}\right)
\end{align*}
where (a) follows from $c_\beta>0$, (b) from $\delta^2<\delta$, (c) from $\log(\iota)<\iota$ (since $\iota>1$ from~\eqref{eq:iota_bound}), and (d) from $\iota\leq \log\left(\frac{(c_\beta+1)dKH}{\delta}\right)$ and from~\eqref{eq:upp_low_bound}.

Now, joining the upper bounds for (I) and (II) in~\eqref{eq:to_upper_bound}, we finally obtain
\begin{equation*}
\norm{\sum_{\tau = 1}^{k-1} \phi^{\tau}_h [V^{i,k}_{h+1}(x^{\tau}_{h+1}) - \Pe_h V^{i,k}_{h+1}(x_h^{\tau}, a_h^{\tau})]}_{(\Lambda^k_h)^{-1}}^2 \leq 
(36+8\log(9)+4\log(8))d^2H^2 \log ((c_\beta+1)dKH/
\delta)
\end{equation*}
which proves the claim and thus the proof.
\end{proof}

\begin{proof}[Proof of Lemma~\ref{lem:basic_relation}]
For any $(i,k)\in[n]\times[K]$,
\begin{align*}
w^{i,k}_h -  w^{i,\bar{\pi}}_h
&= (\Lambda^k_h)^{-1} \sum_{\tau = 1}^{k-1} \phi^{\tau}_h (r^{\tau}_h + V^{i,k}_{h+1}(x^{\tau}_{h+1}))- w^{i,\bar{\pi}}_h  \\
&\overset{\textrm{(a)}}{=} (\Lambda^k_h)^{-1} \sum_{\tau = 1}^{k-1} \phi^{\tau}_h ({\phi^{\tau}_h}^\top w^{i,\bar{\pi}}_h - \Pe_h V^{i,\bar{\pi}}_{h+1}(x_h^{\tau}, a_h^{\tau}) + V^{i,k}_{h+1}(x^{\tau}_{h+1}))- w^{i,\bar{\pi}}_h  \\
&= (\Lambda^k_h)^{-1} \left(
\left(
\sum_{\tau = 1}^{k-1}
\phi^{\tau}_h(\phi^{\tau}_h)^\top
-\Lambda_h^k\right)w^{i,\bar{\pi}}_h \right. \\
&\quad \left.  
 + \sum_{\tau = 1}^{k-1} \phi^{\tau}_h \bigl (V^{i,k}_{h+1}(x^{\tau}_{h+1}) - \Pe_h V^{i,\bar{\pi}}_{h+1}(x_h^{\tau}, a_h^{\tau}) \bigr )\right) \\
&\overset{(b)}{=} (\Lambda^k_h)^{-1} \left(-\lambda w^{i,\bar{\pi}}_h + \sum_{\tau = 1}^{k-1} \phi^{\tau}_h (V^{i,k}_{h+1}(x^{\tau}_{h+1}) - \Pe_h V^{i,\bar{\pi}}_{h+1}(x_h^{\tau}, a_h^{\tau}))\right) \\
&= \underbrace{-\lambda (\Lambda^k_h)^{-1} w^{i,\bar{\pi}}_h}_{\textrm{(I)}} + 
\underbrace{(\Lambda^k_h)^{-1} \sum_{\tau = 1}^{k-1} \phi^{\tau}_h (V^{i,k}_{h+1}(x^{\tau}_{h+1}) - \Pe_h V^{i,k}_{h+1}(x_h^{\tau}, a_h^{\tau}))}_{\textrm{(II)}} \\
&\quad + \underbrace{(\Lambda^k_h)^{-1}\sum_{\tau = 1}^{k-1} \phi^{\tau}_h \Pe_h (V^{i,k}_{h+1} - V^{i,\bar{\pi}}_{h+1})(x_h^{\tau}, a_h^{\tau})}_{\textrm{(III)}}. %
\end{align*}
where (a) follows from the fact that,  for any $(x, a, h) \in \S \times \A \times [H]$, $Q^{i,\bar{\pi}}_h(x, a) :=  \langle \phi(x, a), w^{i,\bar{\pi}}_h\rangle = (r_h + \Pe_h V^{i,\bar{\pi}}_{h+1})(x, a) 
$ for some $w^{i,\bar{\pi}}_h\in\R^d$ (this follows from Proposition~\ref{prop:lin-Q} and the Bellman equation); and (b) follows from the definition of $\Lambda_h^k$. Since $\langle\phi(x, a), w^{i,k}_h\rangle - Q_h^{i,\bar{\pi}}(x, a)=\langle\phi(x, a), w^{i,k}_h-w^{i,\bar{\pi}}_h\rangle$ for any $(x,a)\in\S\times\A$, then we look to bound the inner product of each of the terms (I) -- (III) with the term $\phi(x, a)$.

Regarding the term (I),
\begin{multline*}
|\langle\phi(x,a),\textrm{(I)}\rangle|=|\langle \phi(x, a),\lambda (\Lambda^k_h)^{-1} w^{i,\bar{\pi}}_h\rangle| = |\lambda \langle (\Lambda^k_h)^{-1/2}\phi(x, a), (\Lambda^k_h)^{-1/2} w^{i,\bar{\pi}}_h\rangle|\\
\leq \lambda \norm{w_h^{i,\bar{\pi}}}_{ (\Lambda^k_h)^{-1}} \sqrt{\phi(x, a)^\top (\Lambda^k_h)^{-1}  \phi(x, a)}
\leq \sqrt{\lambda} \norm{w_h^{i,\bar{\pi}}} \sqrt{\phi(x, a)^\top (\Lambda^k_h)^{-1}  \phi(x, a)}
\end{multline*}
where the last inequality follows from $\norm{\,\cdot\,}_{(\Lambda_h^k)^{-1}}\leq \frac{1}{\sqrt{\lambda}}\norm{\,\cdot\,}$.

For the term (II), since the event $\mathcal{E}_i$ from Lemma~\ref{lem:stochastic_term} is given and $\lambda=1$, we directly obtain
\begin{multline*}
|\langle\phi(x,a),\textrm{(II)}\rangle|=\left|\left\langle \phi(x, a), (\Lambda^k_h)^{-1} \sum_{\tau = 1}^{k-1} \phi^{\tau}_h (V^{i,k}_{h+1}(x^{\tau}_{h+1}) - \Pe_h V^{i,k}_{h+1}(x_h^{\tau}, a_h^{\tau}))\right\rangle\right|
\\
\leq
\norm{\sum_{\tau = 1}^{k-1} \phi^{\tau}_h (V^{i,k}_{h+1}(x^{\tau}_{h+1}) - \Pe_h V^{i,k}_{h+1}(x_h^{\tau}, a_h^{\tau}))}_{(\Lambda^k_h)^{-1}}\norm{\phi(x,a)}_{(\Lambda^k_h)^{-1}}\\
\leq 
C dH\sqrt{\log ((c_\beta+1)dKH/\delta)} \sqrt{\phi(x, a)^\top (\Lambda^k_h)^{-1}  \phi(x, a)} 
\end{multline*}
where $C$ is an absolute constant independent of $c_\beta>0$.

For the term (III),
\begin{align*}
\langle\phi(x,a),\textrm{(III)}\rangle&=\left \langle \phi(x, a), (\Lambda^k_h)^{-1}\sum_{\tau = 1}^{k-1} \phi^{\tau}_h \Pe_h (V^{i,k}_{h+1} - V^{i,\bar{\pi}}_{h+1})(x_h^{\tau}, a_h^{\tau}) \right \rangle\\
&= \bigg \langle \phi(x, a), (\Lambda^k_h)^{-1}\sum_{\tau = 1}^{k-1} \phi^{\tau}_h (\phi^{\tau}_h)^\top \int_{\S} (V^{i,k}_{h+1} - V^{i,\bar{\pi}}_{h+1})(x') d \mu_h(x')\bigg\rangle\\
&\overset{\textrm{(a)}}{=} \underbrace{\bigg \langle \phi(x, a), \int_{\S} (V^{i,k}_{h+1} - V^{i,\bar{\pi}}_{h+1})(x') d \mu_h(x')\bigg \rangle}_{\textrm{(III.1)}}\\
&\quad 
\underbrace{-\lambda \bigg\langle \phi(x, a), (\Lambda^k_h)^{-1}\int_{\S} (V^{i,k}_{h+1} - V^{i,\bar{\pi}}_{h+1})(x') d \mu_h(x') \bigg\rangle}_{\textrm{(III.2)}}
\end{align*}
where (a) follows from the definition of $\Lambda_h^k$. We immediately see from our assumption on linear stochastic game that $\textrm{(III.1)}=\Pe_h (V^{i,k}_{h+1} - V^{i,\bar{\pi}}_{h+1})(x, a)$ and 
\begin{multline*}
|\textrm{(III.2)}|
\leq \lambda\norm{\int_\S (V^{i,k}_{h+1} - V^{i,\bar{\pi}}_{h+1})(x')d \mu_h(x')}_{(\Lambda_h^k)^{-1}} \sqrt{\phi(x, a)^\top (\Lambda^k_h)^{-1}  \phi(x, a)}\\
\leq \sqrt{\lambda}\norm{\int_\S (V^{i,k}_{h+1} - V^{i,\bar{\pi}}_{h+1})(x')d \mu_h(x')} \sqrt{\phi(x, a)^\top (\Lambda^k_h)^{-1}  \phi(x, a)}\\
\overset{\textrm{(a)}}{\leq} \sqrt{\lambda}2H\int_\S \norm{\mu_h(x')} dx'  \sqrt{\phi(x, a)^\top (\Lambda^k_h)^{-1}  \phi(x, a)}
\overset{(b)}{\leq} 2 H \sqrt{d\lambda} \sqrt{\phi(x, a)^\top (\Lambda^k_h)^{-1}  \phi(x, a)}
\end{multline*}
where (a) follows from the value functions being bounded, and (b) from the definiton of the linear MDP.

Finally, putting it all together with $\lambda=1$, we conclude that, 
\begin{multline*}
|\langle\phi(x, a), w^{i,k}_h\rangle - Q_h^{i,\bar{\pi}}(x, a)  -  \Pe_h (V^{i,k}_{h+1} - V^{i,\bar{\pi}}_{h+1})(x, a)| \\
\leq \left(\norm{w_h^{i,\bar{\pi}}}+CdH \sqrt{\log ((c_\beta+1)dKH/\delta)}+2H\sqrt{d}\right)\sqrt{\phi(x, a)^\top (\Lambda^k_h)^{-1}  \phi(x, a)}\\
\leq\left(4H\sqrt{d}+CdH \sqrt{\log ((c_\beta+1)dKH/\delta)}\right)\sqrt{\phi(x, a)^\top (\Lambda^k_h)^{-1}  \phi(x, a)}
\end{multline*}
where the last inequality follows from Proposition~\ref{prop:lin-Q}.

Now, from equation~\eqref{eq:upp_low_bound} in Lemma~\ref{lem:stochastic_term}, we have $\sqrt{\log ((c_\beta+1)dKH/\delta)}>1$ independently from $c_\beta>0$, and thus
$$
|\langle\phi(x, a), w^{i,k}_h\rangle - Q_h^{i,\bar{\pi}}(x, a)  -  \Pe_h (V^{i,k}_{h+1} - V^{i,\bar{\pi}}_{h+1})(x, a)| \leq \bar{C} dH\sqrt{\log ((c_\beta+1)dKH/\delta)} \sqrt{\phi(x, a)^\top (\Lambda^k_h)^{-1}  \phi(x, a)},
$$
for an absolute constant $\bar{C}= C+4$ independent of $c_{\beta}$.

Finally, to prove this lemma, we only need to show that there exists a choice of the absolute positive constant $c_\beta$ so that
$
\bar{C}\sqrt{\log ((c_\beta+1)dKH/\delta)}\leq c_\beta \sqrt{\iota}
$, which is equivalent to
\begin{equation} \label{eq:choice_beta_constant}
\bar{C}\sqrt{\iota + \log(c_\beta + 1)} \le c_\beta \sqrt{\iota}
\end{equation}
since $\sqrt{\log\left(\frac{(1+c_\beta)dKH}{\delta}\right)}=\sqrt{\log\left(\frac{dKH}{\delta}\right)+\log(1+c_\beta)}=\sqrt{\iota+\log(1+c_\beta)}$.

Two facts are known: 1) $\iota \in [\log(2), \infty)$ by its definition and $d\geq 2$; and 2) $\bar{C}$ is an absolute constant independent of $c_\beta$. 

Since we know we are looking for  $c_\beta>0$ and using the bound $\log(x)\leq x-1$ for any positive $x\in\R$, we conclude that proving the following equation implies~\eqref{eq:choice_beta_constant},
\begin{equation} \label{eq:choice_beta_constant_2}
\bar{C}\sqrt{\iota+c_\beta}\leq c_\beta\sqrt{\iota}.
\end{equation}
After some algebraic calculations, we can show that
\begin{equation}
\label{eq:c_beta_lower}
c_\beta \geq \frac{\bar{C}^2}{2\log(2)}+\frac{1}{2}\sqrt{\frac{\bar{C}^4}{(\log(2))^2}+4\bar{C}^2}    
\end{equation}
suffices. This finishes the proof.
\end{proof}

\end{document}